%% file: main.tex
\documentclass[11pt]{article}
\usepackage[utf8]{inputenc}
\usepackage[margin=1in]{geometry}

\usepackage[round]{natbib}
\usepackage{amsfonts}
\usepackage{amsthm}

\usepackage{amsmath}
\usepackage{mathtools}
\usepackage{amssymb}
\usepackage{etoolbox}
\usepackage{amssymb}

\usepackage{verbatim}
\usepackage{tabto}

\usepackage{hyperref}
\usepackage{enumitem}

\usepackage{fancyhdr}
\usepackage{xspace}
\usepackage{xcolor}
\usepackage{xparse}

\usepackage[font=small]{caption}

\usepackage{booktabs}
\usepackage{array}
\usepackage{arydshln}
\usepackage[extdef=true]{delimset}
\usepackage[capitalize]{cleveref}
\crefname{claim}{Claim}{Claims}
\crefname{assumption}{Assumption}{Assumptions}

\usepackage{tkthm}
\usepackage{tkdefs}

\usepackage{titlesec}
\titleformat*{\paragraph}{\bfseries\itshape}

\title{Stability vs Implicit Bias of Gradient Methods\\
on Separable Data and Beyond}

\author{%
    Matan Schliserman\thanks{Blavatnik School of Computer Science, Tel Aviv University; \texttt{schliserman@mail.tau.ac.il}.}
    \and%
    Tomer Koren\thanks{Blavatnik School of Computer Science, Tel Aviv University, and Google Research; \texttt{tkoren@tauex.tau.ac.il}.}
}

\date{\today}
\allowdisplaybreaks

\begin{document}
\maketitle

\begin{abstract}%
  An influential line of recent work has focused on the generalization properties of unregularized gradient-based learning procedures applied to separable linear classification with exponentially-tailed loss functions. 
The ability of such methods to generalize well has been attributed to the 
their implicit bias towards large margin predictors, both asymptotically as well as in finite time.
We give an additional unified explanation for this generalization and relate it to two simple properties of the optimization objective, that we refer to as realizability and self-boundedness.
We introduce a general setting of unconstrained stochastic convex optimization with these properties, and analyze generalization of gradient methods through the lens of algorithmic stability. 
In this broader setting, we obtain sharp stability bounds for gradient descent and stochastic gradient descent which apply even for a very large number of gradient steps, and use them to derive general generalization bounds for these algorithms.
Finally, as direct applications of the general bounds, we return to the setting of linear classification with separable data and establish several novel test loss and test accuracy bounds for gradient descent and stochastic gradient descent for a variety of loss functions with different tail decay rates.  In some of these cases, our bounds significantly improve upon the existing generalization error bounds in the literature.%
\end{abstract}

\input{paper}

\input{nonsmooth}
\end{document}

%% file: paper.tex

\newcommand{\D}{\mathcal{D}}
\newcommand{\etabar}{\overline\eta}
\newcommand{\tail}{\epsilon}
\newcommand{\pow}{\delta}
\newcommand{\coef}{c}
\newcommand{\margin}{\gamma}
\newcommand{\polpow}{\alpha}
\newcommand{\wbar}{\overline{w}}
\newcommand{\Fall}[1]{F_{#1}}
\newcommand{\F}{\mathcal{F}}
\newcommand{\E}{\mathbb{E}}
\newcommand{\Ealg}{\mathbb{E}_A}
\newcommand{\Es}{\mathbb{E}_S}
\newcommand{\Esalg}{\mathbb{E}_{S,A}}
\newcommand{\lipstab}{\epsilon_{\mathrm{stab}}^{\smash{(1)}}}
\newcommand{\smostab}{\epsilon_{\mathrm{stab}}^{\smash{(2)}}}
\newcommand{\R}{\mathbb{R}}
\renewcommand{\epsilon}{\varepsilon}
\newcommand{\artz}{z_0}

\section{Introduction}

There is a significant interest nowadays in understanding the generalization
properties of unregularized gradient-based learning procedures.  This has been
motivated by empirical observations in the context of modern deep learning,
where minimizing the training error without any explicit attempt to constrain
model complexity or to avoid overfitting using explicit regularization, often
results with models that generalizes well.  It has been argued that this
phenomenon is explained by the ``implicit bias'' of the learning algorithm,
whereby the dynamics of the optimization method itself serve as a form of
regularization that biases the algorithm towards favorable ``simple'' models
that will not overfit~\citep[e.g.,][]{neyshabur2014search,neyshabur2017exploring}.

In this context, the setting of linear classification with separable data has
attracted particular attention.  A compelling result of
\citet{soudry2018implicit} showed that by simply minimizing the (unregularized)
empirical risk over a linearly-separable training set using plain gradient
descent, the trained predictor asymptotically converges (in direction) to the
max-margin solution over the dataset, provided that we use an
exponentially-tailed classification loss such as the logistic loss. 
Consequently, by virtue of standard margin-based
generalization bounds for linear predictors, we obtain that the gradient descent
solution does not overfit when the number of steps $T$ is sufficiently large,
even though its magnitude (i.e., Euclidean norm) grows indefinitely with $T$
due to the lack of explicit regularization. This result has been later extended
in several ways to accommodate for other optimization algorithms and loss
functions~\citep{ji2018risk,ji2019refined,nacson2019convergence,nacson2019stochastic,ji2020regpath}.

In terms of non-asymptotic rates, the convergence in direction to a large margin
solution established in these works is rather slow and decays only
logarithmically with $T$; the implied generalization error bounds for the
trained predictors thus come into effect only when $T$ is very large, and at
least exponentially large in the size of the training set and in the maximal
margin.
In a more recent work, via a refined analysis of the margins attained by the
normalized iterates of gradient methods, \citet{shamir2021gradient} established
finite-time generalization error bounds that apply to smaller and more realistic
values of $T$.  For example, for gradient descent used to minimize the average
empirical logistic loss over a separable training set of size~$n$ with
margin~$\gamma$, these bounds assume the form
$\wt{O}
(1/\gamma^2 T + 1/\gamma^2
n)$, suppressing logarithmic factors in $T$.  As discussed by 
\citet{shamir2021gradient}, such a bound is essentially optimal to within
these logarithmic factors.

Notably, all of the aforementioned generalization error bounds apply to the
zero-one accuracy of the normalized predictor, rather than to the loss of the
unnormalized model actually being optimized.
%
%
In fact, a main popularized takeaway from this line of work is that
generalization in terms of zero-one prediction accuracy keeps improving as the
number of gradient steps $T$ grows, regardless of the test loss which could
actually increase asymptotically (see a detailed discussion in
\citealp{soudry2018implicit}).  In other words, by inspecting the loss on a
validation set one might be led to think that the learning algorithm is
overfitting as $T$ grows, whereas the underlying predictor actually keeps
improving.
However, the increase in loss \citet{soudry2018implicit} demonstrate
(theoretically) is again only logarithmic in $T$, and thus affects only the very
late stages of optimization.  A natural question is then: are the zero-one
generalization guarantees strictly better than the test loss guarantees also
earlier in the optimization, after a realistic number of steps?
And how do the achievable test loss bounds depend on properties of the
particular loss function?  Even further, could we in fact benefit, in terms of
generalization, by an early-stopping of the optimization process?

In this paper, we give positive answers to these questions and provide sharp
test loss bounds that match, up to logarithmic factors, analogous existing
bounds for the zero-one accuracy.
First, for exponentially-decaying losses such as the logistic loss, we give a
bound of the form $\wt{O}(1/\gamma^2 T + 1/\gamma^2n)$ over the test loss of the
model trained with $T$ steps of either GD or SGD over a separable training set
of size $n$ and margin $\gamma$.
Thus, up to logarithmic factors, the real-valued loss of the trained
(unnormalized) model exhibits the same ideal generalization rates known for the
zero-one prediction accuracy. 
In other words, unless the number of steps $T$ is exponentially large (in
$1/\gamma$ and $n$), the test loss does in fact diminish to zero at a nearly
optimal rate, closely following the zero-one accuracy.

We also establish analogous results for a variety of other loss functions,
including polynomially-tailed, sub-exponentially tailed, and super-exponentially
tailed losses; these are summarized in \cref{table:bounds}. Since the loss
functions we consider are all surrogates of the zero-one loss, these bounds
immediately imply identical bounds for the test zero-one prediction accuracy of
the trained models. 
To the best of our knowledge, these are the first non-asymptotic results for those loss functions in this context. 

In fact, our analysis applies much more broadly than just to separable linear
classification. We consider a general unbounded stochastic convex
optimization problem of the form $\min_w F(w) := \E_{z \sim \D}[f(w,z)]$ where
the objective $f$ is a nonnegative and smooth convex function over $\R^d$, and 
identify two
simple and general conditions on $f$ that enable sharp generalization bounds, as
in the more specific separable linear classification setup, for both gradient
descent (GD) and stochastic gradient descent (SGD).  The first is a
``realizability'' condition that captures the rate at which the tail of $f$
diminishes to zero as the Euclidean norm of $w$ grows; this is merely a
quantitative analogue of the separability assumption that an optimal zero
population loss $F$ is attained ``at infinity.''  
The second condition is a ``self-boundedness'' property of the gradient operator
of $f$, of the form $\norm{\nabla f(w,z)} \leq c (f(w,z))^{1-\delta}$ for all
$w$ and $z$, for some constants $c,\delta \geq 0$. As has been observed in
previous work~\citep{ji2018risk}, while for general nonnegative $1$-smooth functions this
condition holds with $\delta = \smash{\tfrac12}$ (and $c=\sqrt{2}$; see, e.g.,
\citealp{nesterov2003introductory}), exponentially-tailed losses satisfy a stronger version of this relation with $\delta=0$.  More
generally, we observe that a similar self-boundedness condition holds in fact
for a large variety of loss functions, albeit with a possibly different setting
of $\delta$ and $c$.

Given that these realizability and self-boundedness conditions hold, and
assuming $\delta \to 0$ and $c = \Theta(1)$ to simplify this informal
presentation, our generalization loss bounds for both GD and SGD take the form
$
    O( \epsilon + \ifrac{\rho(\epsilon)^2}{T} + \ifrac{\rho(\epsilon)^2}{n} )
$
for an arbitrary $\epsilon > 0$, where $\rho(\epsilon)$ stands for the minimal
possible norm $\norm{w^*}$ of a reference point $w^* \in \R^d$ where the test
loss drops below $\epsilon$.  (The existence of such a point is guaranteed by
our realizability condition.)  For example, in the specific case of separable
linear classification with margin $\gamma$ and an exponentially-tailed loss, it
is not hard to show that there exists a reference point with test loss
$<\epsilon$ and norm $\rho(\epsilon) = O(\ipfrac1\gamma
\log{\ipfrac1\epsilon)}$, which implies the bound we stated earlier for this
case if we set $\epsilon = \Theta(\ifrac{1}{T})$. Similarly, the other results in
\cref{table:bounds} follow directly from this general bound, up to small
variations in the self-boundedness parameters $\delta$ and $c$, and whether or
not the loss function is Lipschitz on the entire space.

To prove these results, we appeal to algorithmic stability arguments~\citep{bousquet2002stability,shalev2009stochastic}.
Typical stability bounds for gradient methods degrade (at least) linearly with
the number of gradient
steps~\citep{hardt2016train,chen2018stability,bassily2020stability,attia2021algorithmic}.
We show that under the self-boundedness condition with a sufficiently small~$\delta$, 
a significantly more moderate degradation in stability is possible.
Notably, our stability bounds scale roughly with the optimization error of the
gradient method iterates, so that a smaller training loss directly translates to
a better stability bound, and in turn, to a tighter bound on the generalization
gap.  
Thus, in contrast to traditional generalization error analyses that aim to
strike a balance between the training error and the generalization gap, our
analysis shows that for self-bounded objectives (with sufficiently small $\delta$) 
these in fact do not stand at odds with each other, and improving the former 
directly leads to stronger generalization guarantees.
We remark that recently \citet{lei2020fine} established refined stability bounds under 
a realizability assumption similar to ours; however, since they do not make use of the stronger self-boundedness condition, their bounds degrades quickly (roughly linearly) with the number of gradient steps.

\begin{table}[t]
\small
\begin{center}
\begin{tabular}{lccccc} 
\toprule
{\bfseries Loss function}
& {\bfseries Tail behavior}
& {\bfseries $\boldsymbol{ \rho(\epsilon) }$}
& {\bfseries $\boldsymbol{ (\coef,\pow) }$}
& {\bfseries Generalization loss}

\\[0.5ex] 
\toprule
Logistic  
& $\exp(-x)$
&$\tfrac{1}{\margin} \log(\frac{1}{\tail})$
& $(1,0)$
& 
$O\left( \frac{\log^2(T)}{\margin^2T}+
     \frac{ \log^2(T)}{\margin^2n}\right)$
\\ 
Multi-class logistic  
& $\exp(-x)$  
&$\tfrac{1}{\margin} \log(\frac{K}{\tail})$
& $(2,0)$
& 
$O\left( \frac{\log^2(KT)}{\margin^2T}+
     \frac{ \log^2(KT)}{\margin^2n}\right)$
\\
Polynomial  
& $x^{-\polpow}$~,~$\polpow > 0$ 
&$\tfrac{1}{\margin} \tail^{-\ifrac{1}{\polpow}}$
& $(\polpow,0)$
& $O\brk2{\brk!{\frac{\polpow}{\margin}}^{\frac{2\polpow}{2+\polpow}}
    \brk2{\frac{1}{T^{\frac{\polpow}{2+\polpow}}}+\frac{T^{\frac{2}{2+\polpow}}}{n}}}$  
\\
Sub-exponential  
&$\exp(-x^{\polpow})$~,~$\polpow < 1$ 
&$\tfrac{1}{\margin} \log^{\ifrac{1}{\polpow}}(\frac{1}{\tail})$
& $(\polpow,0)$
& 
$O\left(\frac{\polpow \log^{\frac{2}{\polpow}}(T)}{\margin^2 T}+\frac{\polpow^2\log^{\frac{2}{\polpow}}(T)}{\margin^2 n}\right)$
\\
Super-exponential
&$\exp(-x^{\polpow})$~,~$\polpow> 1$
&$\tfrac{1}{\margin} \log^{\ifrac{1}{\polpow}}(\frac{1}{\tail})$
&$(\frac{\polpow}{2\pow},\pow)$~,~$\forall \pow>0$
& 
$O\left(\frac{\polpow^2 \log^{\frac{2}{\polpow}}(T)}{\margin^2 T}+\frac{\polpow^{3} \log^{\frac{2+\polpow}{\polpow}}(T)}{\margin^{2} n}\right)$
\\
Probit
& $\exp(-x^2)$
&$\tfrac{1}{\margin} \log^{\ifrac{1}{2}}(\frac{1}{\tail})$
&$(\frac{3}{\pow},\pow)$~,~$\forall \pow>0$
& 
$O\left(\frac{\log(T)}{T\margin^2}
    +  \frac{\log(T)^4}{n\margin^{4}}\right)$
\\
\bottomrule
\end{tabular}
\end{center}
\caption{Generalization loss bounds established in this paper for constant step-size GD and SGD on separable linear classification, for loss functions which satisfy \cref{ass:tail,ass:grad_norm_pow} along with the relevant parameters.
Here, $T$ is the number of gradient steps, $n$ is the size of the training set, and $\gamma$ is the margin over the training set. (For the multi-class logistic loss, $K$ is the number of classes.)}
\label{table:bounds}
\end{table}

\subsection{Summary of contributions}
\label{contributions}
To summarize, the main contribution of this paper are as follows: 

\begin{itemize}
    \item We introduce a general setting of unconstrained stochastic convex
    optimization, that captures the well-studied setting of unregularized (and
    unconstrained) separable linear classification as a special case (see
    \cref{sec:setup}). In this framework, we identify two simple
    conditions---realizability and self-boundedness---under which strong
    generalization loss bounds are attainable.
    \item We obtain sharp algorithmic stability bounds for gradient descent and
    stochastic gradient descent in the general unconstrained stochastic convex
    optimization setup, which directly translate to bounds on the generalization
    gap (in terms of actual loss values) of the trained model (see \cref{exp
    stab power,stab_gd_nonlip,exp stab power sgd} in \cref{sec:stability}). Our
    stability bounds scale with the optimization error of the gradient method
    iterates and allows for a favorable trade-off between training error and
    generalization gap, which applies even after a very large number of gradient
    steps.
    \item Combined with standard convergence bounds for smooth optimization,  we
    derive generalization loss bounds for gradient descent and stochastic
    gradient descent in the general stochastic convex optimization setting (see
    \cref{main_thm_GD,main_thm_sgd} in \cref{sec:bounds}).
    \item Finally, as direct applications of the general bounds, we also achieve
    several new generalization results in the setting of unregularized linear
    classification with separable data (see \cref{sec:examples}).  
    We consider a variety of loss functions with different tail decay rates, and
    establish test loss bounds for gradient descent and stochastic gradient
    descent for the corresponding classification problem.  These bounds aret
    summarized in \cref{table:bounds}, and to the best of our knowledge, are new
    to the literature.
\end{itemize}

\subsection{Discussion and some implications}
\label{implications}
\paragraph{On the role of early stopping:}

For polynomially-tailed losses, with a tail decay rate of $x^{-\alpha}$ ($\alpha
\geq 1$), our generalization error bounds are optimized when $T=n$, and degrade
in quality polynomially-fast with $T$. Interestingly, the test performance
implied by our upper bounds after $T=n$ steps of GD is strictly better than the
performance of the asymptotic direction to which GD converges as $T \to \infty$:
while \cite{ji2020regpath} show that the asymptotic margin is in general
$\Theta(n^{-\ifrac{1}{(\alpha+1)}})$ (see their Proposition 12), which by
standard margin bounds translates to a generalization error bound of
$\Theta(n^{-\ifrac{(\alpha-1)}{(\alpha+1)}})$, we obtain a test error bound of
$O(n^{-\ifrac{\alpha}{(\alpha+2)}})$, which is always better than the former
bound---the best that one can achieve without early stopping.%
\footnote{For this discussion, we treat both the max-margin $\gamma$ and the decay exponent $\alpha$ as constants; see the bounds in \cref{table:bounds} for the precise dependence on these parameters.}
(E.g., for $\alpha=1$ the asymptotic bound becomes vacuous, while our bound for $T=n$ is $O(n^{-\ifrac13})$.)
%
%
This addresses open questions from earlier works (e.g., \cite{ji2020regpath}
mention this explicitly) regarding the role of early stopping: we see that with
a polynomially-decaying loss, not only that early stopping does not harm test
performance, but it can actually strictly improve the latter.

\paragraph{On the optimality of exponential tails:}

Our results shed some light on what enables exponentially-tailed losses in
attaining optimal test loss bounds, and thus provides a different justification
than that of \citet{ji2020regpath} for this widely popular choice of a loss
function.  Crucially, the rapid decay of an exponential tail allows for a faster
optimization rate, as the training loss drops below $\epsilon$ at distance from
origin of at most $O(\log(\ifrac1\epsilon))$ as a function of $\epsilon$.  As
already discussed, this faster optimization rate directly translates to a
tighter generalization loss bound. 
It would be tempting to conclude that decay rate even faster than exponential
would lead to better bounds.  However, our analysis of super-exponentially
decaying tails, e.g., tails of the form $\exp(-x^\alpha)$ for $\alpha>1$,
suggests that faster decaying tails actually allow for weaker stability bounds
that become worse by a factor of $\alpha$.  Thus, somewhat remarkably, an
exponentially-decaying tail strikes a fine balance between optimization rate and
model stability.

\paragraph{On implicit bias vs.\ algorithmic stability:}

It is interesting to relate two very different aspects that govern 
generalization in the context of separable linear classification: implicit bias and algorithmic stability.  Our results suggest that while the effect of implicit biases of gradient methods (towards large margin solutions) kicks in rather late in the optimization, stability seems to control generalization more tightly earlier on.  Intriguingly, for exponentially-tailed losses that have been the extensive focus of previous work, the stability and implicit bias (i.e., margin based) arguments both imply the same tight generalization rates for the downstream prediction accuracy of the trained model.  However, this does not seem to be the case more generally, e.g., in the case of polynomially-tailed losses, where generalization bounds based on stability are strictly better than those obtained from margin-based arguments: indeed, the analysis of \citet{shamir2021gradient} in this case yielded suboptimal bounds compared to the asymptotic bounds of \citet{ji2020regpath} which, as we have just discussed, are weaker than our finite-time bounds based on stability.

\section{Unconstrained Stochastic Convex Optimization}
\label{sec:setup}

We consider the following setting of unconstrained and unregularized stochastic convex
optimization over $\R^d$, that subsumes separable linear classification as a
special case.
Let $\D$ be a distribution over a probability space $\mathcal{Z}$. We measure
the loss of a model $w$ on an example $z$ by a loss function $f(w,z)$ where
$f:\R^d \times \mathcal{Z} \to \R$ be a positive, convex and $L$-smooth function
(with respect to $w$, for all $z \in \mathcal{Z}$). We are interested to find a
prediction $w$ which has a small \textit{population loss}, defined as 
\begin{equation*}
    F(w)=\E_{z\sim\mathcal{D}}[f(w,z)] .
\end{equation*}
For finding such a model $w$, we use a set of training examples $S=\{z_1,...,z_n\}$ which drawn i.i.d.\ from $\D$ and an empirical proxy, the \textit{empirical loss}, defined as
\begin{equation}
\label{F_hat_def}
    \widehat{F}(w)
    =
    \frac1n \sum_{i=1}^n f(w,z_i)
    .
\end{equation}

\subsection{The Realizability and Self-boundedness conditions}
\label{sec:self_bound}

We additionally make the following two key assumptions on the objective $f$:
 
\begin{assumption}[Realizability] \label{ass:tail}
There exists a monotonic decreasing function $\rho : \R^+ \to \R^+$ such that
for every $\tail > 0$, there exists $w^* \in \R^d$ such that $f(w^*,z)\leq
\tail$ for all $z$, and $\|w^*\|\leq \rho(\tail)$.
\end{assumption}

\begin{assumption}[Self boundedness] \label{ass:grad_norm_pow}
There exist $\coef>0$ and $0\leq \pow\leq \tfrac12$, such that $\|\nabla
f(w,z)\|\leq \coef f(w,z)^{1-\pow}$ for all $w$ and $z$.
\end{assumption}

As discussed in the introduction, \cref{ass:tail} is a quantitative abstraction of the separability assumption in linear classification that an optimal zero population loss is attained ``at infinity.''
Without loss of generality, we will assume that $\rho(\epsilon) \geq 1$ for all $\epsilon$; otherwise, we may replace $\rho$ with the function $\epsilon \to \max\{\rho(\epsilon),1\}$.

\cref{ass:grad_norm_pow} on the other hand is the main assumption that enables our algorithmic stability and generalization results.  
We recall that the self-boundedness condition always holds with $c=\sqrt{2\beta}$ and $\delta=\tfrac{1}{2}$ for nonnegative $\beta$-smooth objectives~\citep[e.g.,][]{nesterov2003introductory}. However, we observe that for many loss functions a significantly stronger version of this condition holds with $\delta \ll \tfrac12$, which allows for stronger generalization bounds; see the examples in \cref{table:bounds}, and more details in \cref{sec:examples}. 
The generality of the self-boundedness condition can be seen from the following simple observation:
%
%

\begin{lemma} \label{lem:grad_power_equiv}
Let $f : \R^d \to \R^+$ be a positive and differentiable function. Then,
\begin{enumerate}[label=(\roman*)]
    \item $f$ satisfies \cref{ass:grad_norm_pow} with $\pow=0,c>0$ if and only if $\log(f)$ is $\coef$-Lipschitz;
    \item $f$ satisfies \cref{ass:grad_norm_pow} with $\pow>0,c>0$ if and only if $f^\pow$ is $\coef\pow$-Lipschitz.
\end{enumerate}
\end{lemma}

For example, it is not hard to see that the exponential loss $f(x)=e^{-x}$ and the logistic loss admit the first property above for $\coef=1$; and that a polynomial loss $f(x)=x^\polpow$ ($\polpow\geq 1$) admits the second property for  $\pow=\ifrac{1}{\polpow}$ and $\coef=\polpow$.



\subsection{Example: linear classification with margin}
\label{discuss_clasification}

The setting of linear classification with separable data is obtained as a
special case of our main setting.  
We mainly consider binary classification
with
margin $\gamma>0$, where the examples are
labeled examples $(x_i,y_i)$ and there exists a unit vector $w^*$ such that $y w^* \cdot x
\geq \margin$ almost surely (with respect to a distribution $\D$ over instances $(x,y)$).  
For brevity, we denote $z = y \cdot x$
and assume that data is scaled so that $\norm{z} \leq 1$ with probability $1$.
In this case, we have $f(w,z_i) = \ell(w \cdot z_i)$ for all $i$, where $\ell$ is
a convex, positive and monotonically decreasing loss function, such as the logistic loss $\ell(x) = \log(1+\exp(-x))$. For such functions, \cref{ass:tail} holds:

\begin{lemma} \label{lem:tail_sep}  
If $\ell : \R \to \R^+$ is  positive, convex and strictly monotonically decreasing such that $\lim_{x \to \infty} \ell(x) = 0$. Then, the inverse $\ell^{-1}$ is a well defined function and
$f(w,z) = \ell(w \cdot z)$ satisfies \cref{ass:tail} with $\rho(\epsilon) =
\ipfrac1\margin \ell^{-1}(\epsilon)$.
\end{lemma}

\cref{ass:grad_norm_pow} also holds for many asymptotically-vanishing loss functions.
As discussed in more detail in \cref{sec:examples}, for
exponentially, sub-exponentially, and polynomially tailed losses it is satisfied for $\pow=0$;
for super-exponentially tailed losses (including the Probit loss) \cref{ass:grad_norm_pow} is satisfied for any $\pow>0$, albeit with $c = \Theta(1/\delta)$.
\cref{table:bounds} summarizes the relevant parameters $\rho,\coef,\pow$ for the various loss functions.

\section{Stability bounds for self-bounded objectives}
\label{sec:stability}

For bounding the generalization gap we use algorithmic stability. 
In this section we define the notion of stability that we use, on-average-leave-one-out model stability and prove that popular gradient methods, like gradient descent and stochastic gradient descent, are stable with respect to this notion.

\subsection{Preliminary: On-Average Leave-One-Out model stability}

We first introduce the stability notions we will rely on, and note their connection to generalization.
Specifically, that notion of stability that we consider is on-average leave-one-out
(on-average-loo) model stability. For this definition, we assume without loss of generality that there exists an example $\artz \in \mathcal{Z}$ for which $f(w,\artz)=0$ for all $w$. (Otherwise, we can artificially augment the sample space with such an instance.)
Now, given an i.i.d.~sample $S = (z_1,\ldots,z_n)$, with the corresponding
empirical risk (\cref{F_hat_def}), we define the leave-one-out samples
$
    S_i = (z_1,\ldots,z_{i-1},z_0,z_{i+1},\ldots,z_n)
$
for all $i \in [n]$, 
with the corresponding empirical risks:
\begin{align*}
    \forall ~ i \in [n],
    \qquad
    \widehat{F}_i(w)
    =
    \frac1n \sum_{z \in S_i} f(w,z)
    =
    \frac1n \sum_{j \neq i} f(w,z_j)
    .
\end{align*}
We can now define the on-average-loo model stability for learning algorithms.
This notion of stability amounts to the distance (in Euclidean norm) between the output models of a learning algorithm on the the full sample $S$ and the leave-one-out samples $S_i$, averaged over the choice of $i$.
We refer to two notions of loo-model-stability $\ell_1$-on-average-loo model stability, which will be used for achieving a generalization bound for Lipschitz functions, and $\ell_2$-on-average-loo model stability, we be used for non-Lipschitz functions. 

Throughout the paper, we bound the expected stability and generalization of learning algorithms. We denote by $\Es$ the expectation with respect to the sample $S\sim\mathcal{Z}^n$, and by $\Ealg$ the expectation with respect to the randomization in the learning algorithm (if it is randomized).

\begin{definition}[on-average leave-out-out model stability]
Let $A : \mathcal{Z}^n \to \R^d$ be a learning algorithm. We say that $A$ is
$\ell_1$-on-average-loo model $\epsilon$-stable if for \emph{any} sample $S$, 
\begin{equation} \label{epsilon_l1_stab}  
    \frac1n \sum_{i=1}^n \Ealg\brk[s]!{ \|A(S)-A(S_i)\| }
    \leq 
    \epsilon
    ,
\end{equation}
and that $A$ is $\ell_2$-on-average-loo model $\epsilon$-stable if for any
sample $S$,
\begin{equation} \label{epsilon_l2_stab}  
    \frac1n \sum_{i=1}^n \Ealg\brk[s]!{ \|A(S)-A(S_i)\|^2 }
    \leq 
    \epsilon
    .
\end{equation}
We will denote by $\lipstab$ the infimum over all $\epsilon$ for which \cref{epsilon_l1_stab} holds, and by $\smostab$ the infimum over all $\epsilon$ for which \cref{epsilon_l2_stab} holds.
\end{definition}
 
Note that for any algorithm $A$, the output $A(S_i)$ given the sample $S_i$
viewed as a random variable that is independent of the sample point $z_i$. We
can use this fact to get a connection between on-average-loo model stability and
generalization.
\begin{lemma} \label{lem:lip_gen}  
 Let $A$ be an $\ell_1$-$\epsilon$-on-average-loo model stable learning algorithm. Then, if for every $z$, $f(w,z)$ is convex and $G$-Lipschitz with respect to $w$, 
 \begin{equation*}
     \Esalg\brk[s]!{ F(A(S)) } 
     \leq 
     \Esalg\brk[s]!{ \widehat{F}(A(S)) } + 2G\epsilon 
     .
\end{equation*}
\end{lemma}

\begin{lemma} \label{lem:non_lip_gen}
Let $A$ be an $\ell_2$-$\epsilon$-on-average-loo model stable learning algorithm. Then, if for every $z$, $f(w,z)$ is convex and $L$-smooth with respect to $w$, 
\begin{align*}
    \Esalg\brk[s]!{ F(A(S)) }
    \leq
    4\Esalg\brk[s]!{ \widehat{F}(A(S)) }
    + 3L\epsilon
    .
\end{align*}
\end{lemma}

\subsection{Gradient Descent}
First, we establish a stability bound for GD under \cref{ass:grad_norm_pow}. We consider standard gradient descent (GD) with a fixed step size $\eta>0$ applied to the empirical risk $\widehat{F}$; this method is initialized at a point $w_1$ and at each step $t=1,\ldots,T$ performs an update

\begin{equation} \label{gd_update_rule}  
    w_{t+1} 
    = 
    w_t - \eta \nabla \widehat{F}(w_t)
    .
\end{equation}
The algorithm returns the final model, $w_T$.

%

We begin with bounding the $\ell_1$-on-average-loo model stability of GD with 
self-boundedness.
\begin{theorem} \label{exp stab power}  
Suppose \cref{ass:grad_norm_pow} holds and $f(w,z)$ is positive, convex and
$L$-smooth with respect to $w$, for all $z$.  Then,
GD with step size $\eta \leq \ifrac{2}{L}$ is $\ell_1$-on-average-loo model stable with
\begin{equation*}
    \lipstab
    \leq 
    \frac{\coef \eta T^\pow}{n} \left(\sum_{t=1}^{T}\widehat{F}(w_t)\right)^{1-\pow}
    .
\end{equation*}
\end{theorem}

We will be mostly interested in the regime where \cref{ass:grad_norm_pow} holds with a sufficiently small~$\delta$ so that $\coef$ and $T^\pow$ are constants (namely, where $\delta = O(\ifrac{1}{\log{T}})$). In this
case, we obtain a stability bound of the form $\lipstab = O(\frac{1}{n} \sum_{t=1}^{T} \widehat{F}(w_t))$, which increases very moderately with the number of steps $T$ provided that GD properly minimizes the training error.

In our stability analysis below, we will use the following standard lemma in smooth
convex optimization~\citep[e.g.,][]{hardt2016train}.

\begin{lemma} \label{lem:non_expand_smooth}  
If $f : \R^d \to \R$ is convex and $L$-smooth and $0 < \eta \leq \ifrac{2}{L}$, then for every $u,v \in \R^d$,
\begin{equation*}
    \|(u-\eta \nabla f(u))-(v-\eta \nabla f(v))\|
    \leq 
    \|u-v\|
    .
\end{equation*}
\end{lemma}

\begin{proof}[of \cref{exp stab power}]
Let $\{w_t\}_t$ be the iterates of GD on $\widehat{F}$ and $\{w_t^i\}_t$ be the iterates of GD on $\widehat{F}_i$.
Then for every $t$,
\begin{equation*}
    w_{t+1}^i=w_t^i - \frac{\eta}{n}\sum_{j\neq i}\nabla f(w_t^i,z_j)
    .
\end{equation*}
As a result, by  \cref{lem:non_expand_smooth},
 \begin{align*}
\|w_{t+1} - w_{t+1}^i\|
&=
\norm3{ w_t-\frac{\eta}{n}\sum_{j=1}^n\nabla f(w_t,z_j) - w_t^i + \frac{\eta}{n}\sum_{j\neq i}\nabla f(w_t^i,z_j) }
\\
&=
\norm3{ \frac{1}{n}\sum_{j\neq i} \brk!{w_t-\eta\nabla f(w_t,z_j) - w_t^i + \eta\nabla f(w_t^i,z_j) } + \frac{1}{n}w_t-\frac{1}{n}w_t^i-\eta \nabla f(w_t,z_i)) }
\\
&\leq
\frac{1}{n}\sum_{j\neq i} \norm!{ w_t-\eta\nabla f(w_t,z_j) - w_t^i + \eta\nabla f(w_t^i,z_j) } + \frac{1}{n} \norm!{ w_t-\eta\nabla f(w_t,z_i) - w_t^i }
\\
&\leq
\frac{1}{n}\sum_{j\neq i}^n \|w_t - w_t^i \| + \frac{1}{n}\|w_t-w_t^i\| + \frac{\eta}{n}\|\nabla f(w_t,z_i)\|
.
\end{align*}
By \cref{ass:grad_norm_pow} we get,
\begin{equation}
\label{eq:w_t_form}
    \|w_{t+1}-w_{t+1}^i\| \leq \|w_t-w_t^i\| + \frac{\eta\coef}{n}f(w_t,z_i)^{1-\pow}
    .
\end{equation}
Now, the function $g(x)=x^{1-\pow}$ is concave, thus, by Jensen's inequality,
 \begin{align*}
\frac{1}{n}\sum_{i=1}^n\|w_{t+1} - w_{t+1}^i\|
&\leq
\frac{1}{n}\sum_{i=1}^n\left(\|w_t-w_t^i\| + \frac{\eta\coef}{n}f_i(w_t)^{1-\pow} \right)
\\
&\leq
\frac{1}{n}\sum_{i=1}^n\|w_t-w_t^i\| + \frac{\eta\coef}{n}\frac{1}{n}\sum_{i=1}^n f_i (w_t)^{1-\pow}
\\
&\leq 
\frac{1}{n}\sum_{i=1}^n\|w_t-w_t^i\| + \frac{\eta\coef}{n}\widehat{F}(w_t)^{1-\pow}.
\end{align*}    
By summing over $t=1,\ldots,T-1$ and another use of Jensen inequality, we get,
\begin{align*}
    \frac{1}{n}\sum_{i=1}^n\|w_{T} - w_{T}^i\|
    &\leq 
    \frac{\eta\coef}{n}\sum_{t=1}^{T-1}\widehat{F}(w_t)^{1-\pow}
    \\
    &\leq
    \frac{\eta\coef}{n}\sum_{t=1}^{T}\widehat{F}(w_t)^{1-\pow}
    \\
    &\leq
    \frac{\eta T\coef}{n}\left(\frac{1}{T}\sum_{t=1}^{T}\widehat{F}(w_t)\right)^{1-\pow}
    \\
    &=
    \frac{\eta T^\pow\coef}{n}\left(\sum_{t=1}^{T}\widehat{F}(w_t)\right)^{1-\pow}
    .
\end{align*}
\end{proof}

Now we bound the $\ell_2$-on-average-loo model stability of Gradient Descent on self-bounded loss functions. The resulting bound is stated in the following theorem.

\begin{theorem}
\label{stab_gd_nonlip}
 Suppose \cref{ass:grad_norm_pow} holds. If, for every $z$, $f(w,z)$ is positive, $L$-smooth, convex with respect to $w$, then, GD with step size $\eta \leq \ifrac{2}{L}$ is $\ell_2$-on-average-loo model stable with
\begin{equation*}
\smostab
    \leq 
    \frac{c^2\eta^2 T^{2\pow}}{n^{1+2\pow}} \brk3{ \sum_{t=1}^{T} \widehat{F}(w_t) }^{2(1-\pow)}
    .
\end{equation*}
\end{theorem}

Again, we will be mostly interested in the regime where $\coef$ and $T^\pow$ are constants. In this case, we get a stability bound of the form $O(\frac{1}{n} (\sum_{t=1}^{T} \widehat{F}(w_t))^2)$, which again scales favorably with the number of steps $T$ given that GD is properly tuned for convergence on the empirical risk.


\begin{proof}
Let $\{w_t\}_t$ be the iterates of GD on $\widehat{F}$ and $\{w_t^i\}_t$ be
the iterates of GD on $\widehat{F}_i$. By a similar derivation as in the proof
of~\cref{exp stab power},
\begin{equation}
\label{eq_l2_dist}
    \|w_{t+1}-w_{t+1}^i\| \leq \|w_t-w_t^i\| + \frac{\eta\coef}{n}f(w_t,z_i)^{1-\pow}
    .
\end{equation}
Then, by summing \cref{eq_l2_dist} over $t=1,\ldots,T-1$, we get
\begin{align*}
    \norm{w_T^i-w_T}
    \leq 
    \frac{\eta \coef}{n}\sum_{t=1}^{T-1} f(w_t,z_i)^{1-\pow}
    \leq
    \frac{\eta \coef}{n}\sum_{t=1}^{T} f(w_t,z_i)^{1-\pow}
    .
\end{align*} 
As a result, 
\begin{align*}
    \frac{1}{n} \sum_{i=1}^n \norm{w_T^i-w_T}^2
    &\leq
    \frac{c^2\eta^2 }{n^3} \sum_{i=1}^n \, \brk3{ \sum_{t=1}^{T} f(w_t,z_i)^{1-\pow}}^2
    \\
    &=
    \frac{c^2\eta^2 T^2}{n^3} \sum_{i=1}^n \, \brk3{ \frac{1}{T} \sum_{t=1}^{T} f(w_t,z_i)^{1-\pow}}^2
    \\
    &\leq
    \frac{c^2\eta^2 T^{2\pow}}{n^3} \sum_{i=1}^n \, \brk3{ \sum_{t=1}^{T} f(w_t,z_i) }^{2(1-\pow)}
    \tag{Concavity of $x \to x^{1-\pow}$}
    \\
    &\leq
    \frac{c^2\eta^2 T^{2\pow}}{n^3} \brk3{ \sum_{t=1}^{T} \sum_{i=1}^n f(w_t,z_i) }^{2(1-\pow)}
    \tag{$\pow \leq \frac{1}{2}$,$f(w_t,z_i)\geq 0$}
    \\
    &=
    \frac{c^2\eta^2 T^{2\pow}}{n^{1+2\pow}} \brk3{ \sum_{t=1}^{T} \widehat{F}(w_t) }^{2(1-\pow)}
    .
\end{align*}
\end{proof}

\subsection{Stochastic Gradient Descent}

We now turn to establish a stability bound for Stochastic Gradient Descent (SGD)
under \cref{ass:grad_norm_pow}. Given a dataset $S$ of size $n$, SGD is
initialized at a point $w_1 \in \R^d$ and at each step $t=1,\ldots,T$, samples
randomly an index $i_t\in[n]$ and performs an update
\begin{equation}
\label{sgd_update_rule}
    w_{t+1} = w_t - \eta \nabla f(w_t,z_{i_t}),
\end{equation}
where $\eta>0$ is the step size of the algorithm. We consider a standard variant of SGD that returns the average iterate, namely $\overline{w}_T=\frac{1}{T}\sum_{t=1}^{T}w_t$.

We discuss the performance of SGD that runs for $T$ iterations on data set with size $n$,
and show $\ell_1$-on-average-loo model stability bound for the algorithm.
The stability bound that we get is identical to the  $\ell_1$-on-average-loo model stability bound of GD and stated in the following theorem.

\begin{theorem}
\label{exp stab power sgd}
  Suppose \cref{ass:grad_norm_pow} holds. If, for every $z$, $f(w,z)$ is positive, $L$-smooth, convex with respect to $w$, then, SGD with step size $\eta \leq \ifrac{2}{L}$ is $\ell_1$-on-average-loo model stable with
 \begin{equation*}
    \lipstab
     \leq 
     \frac{\eta T^\delta\coef}{n}\left(\sum_{t=1}^{T}\widehat{F}(w_t)\right)^{1-\pow}
     .
     \end{equation*}
\end{theorem}

\begin{proof}
Let $\{w_t\}_t$ be the iterates of SGD on $S$ and $\{w_t^i\}_t$ be the iterates of SGD on $S_i$.
Then for every $t$,
\begin{equation*}
    w_{t+1}^i=\left\{\begin{array}{cc}
        w_t^i -\eta \nabla f(w_t^i,z_{i_t}) &   i_t\neq i\\
        w_t^i & i_t=i
    \end{array}\right.
\end{equation*}
As a result, by \cref{lem:non_expand_smooth}, if $i_t\neq i$
\begin{align*}
    \|w_{t+1} - w_{t+1}^i\|
    =
    \|w_t-\eta \nabla f(w_t,z_{i_t}) - w_t^i + \eta\nabla f(w_t^i,z_{i_t})\|
    \leq \|w_{t} - w_{t}^i\|
    .
\end{align*}
If $i_t=i$, we have
\begin{align*}
    \|w_{t+1} - w_{t+1}^i\|
    =
    \| w_t-\eta\nabla f(w_t,z_i) - w_t^i \|
    \leq
    \|w_t-w_t^i\| + \eta\coef f(w_t,z_i)^{1-\pow} 
    .
\end{align*}
Then, for every $i$,
\begin{align*}
    \|w_{t+1} - w_{t+1}^i\|
    \leq 
    \eta \coef \sum_{i_j=i,j\leq t} f(w_t,z_i)^{1-\pow} 
    .
\end{align*}
Averaging until $T-1$ and all $i$, we get, for $T\geq 2$,
\begin{align*}
    \frac{1}{nT}\sum_{i=1}^n\sum_{t=1}^{T}\|w_{t} - w_{t}^i\|&=
    \frac{1}{nT}\sum_{i=1}^n\sum_{t=2}^{T}\|w_{t} - w_{t}^i\|
    \\&\leq
    \frac{\eta\coef}{nT} \sum_{i=1}^n\sum_{t=1}^{T-1}\sum_{j\leq t,i_j=i} f(w_j,z_i)^{1-\pow} 
    \\
    &=
    \frac{\eta\coef}{nT}
    \sum_{t=1}^{T-1}\sum_{j\leq t} f(w_j,z_{i_j})^{1-\pow}
    \\
    &\leq
    \frac{\eta\coef}{nT}
    \sum_{t=1}^{T}\sum_{j\leq t} f(w_j,z_{i_j})^{1-\pow}
    .
\end{align*}
We notice that for every $j$, $w_j,z_{i_j}$ are independent. As a result, with
taking expectation on the randomnesses of the algorithm, by Jensen's inequality,
\begin{align*}
    \Ealg\brk[s]4{\frac{1}{n}\sum_{i=1}^n \|\overline{w}_T-\overline{w}_T^i\|}
    &\leq
    \Ealg\brk[s]4{\frac{1}{nT}\sum_{i=1}^n\sum_{t=1}^{T}\|w_{t} - w_{t}^i\|}
    \\
    &\leq 
    \Ealg\brk[s]4{\frac{\eta\coef}{nT} \sum_{t=1}^{T}\sum_{j\leq t} f(w_j,z_{i_j})^{1-\pow}}
    \\
    &\leq
    \frac{\eta\coef}{nT}
    \sum_{t=1}^{T}\sum_{j\leq t} (\Ealg f(w_j,z_{i_j}))^{1-\pow}
    \\
    &\leq
    \frac{\eta\coef}{nT}
    \sum_{t=1}^{T}\sum_{j\leq t} \widehat{F}(w_j)^{1-\pow}
    .
\end{align*}
Now we notice that each term $\widehat{F}(w_t)^{1-\pow}$ appears in the
summation exactly $T-t+1$ times. As a result, 
\begin{align*}
    \Ealg\brk[s]4{\frac{1}{n}\sum_{i=1}^n \|\overline{w}_T-\overline{w}_T^i\|}
    &\leq
    \frac{\eta\coef}{nT}
    \sum_{t=1}^{T}(T-t+1) \widehat{F}(w_t)^{1-\pow} 
    \\
    &\leq 
    \frac{\eta\coef}{n}
    \sum_{t=1}^{T} \widehat{F}(w_t)^{1-\pow}.
    \end{align*}
By Jensen Inequality,
\begin{align*}
    \Ealg\brk[s]4{\frac{1}{n}\sum_{i=1}^n \|\overline{w}_T-\overline{w}_T^i\|}
    &\leq
    \frac{\eta\coef T}{n}\cdot 
    \frac{1}{T}\sum_{t=1}^{T} \widehat{F}(w_t)^{1-\pow} 
    \\
    &\leq 
    \frac{\eta\coef T}{n}
    \left(\frac{1}{T}\sum_{t=1}^{T} \widehat{F}(w_t)\right)^{1-\pow} 
    \\
    &\leq 
    \frac{\eta\coef T^\pow}{n}
    \left(\sum_{t=1}^{T} \widehat{F}(w_t)\right)^{1-\pow} 
    .
\end{align*}
\end{proof}

\section{Generalization loss bounds}
\label{sec:bounds}

We now establish new generalization bounds for gradient methods on self-bounded objectives, using the stability bounds developed in the previous section. 

\subsection{Gradient Descent}

We begin with GD and show the following generalization bound:
\begin{theorem} \label{main_thm_GD}  
Suppose that for any $z$, the loss function $f(w,z)$ is nonnegative, convex,
$L$-smooth with respect to $w$, and further satisfies
\cref{ass:tail,ass:grad_norm_pow}. Then the output of GD with step size
$\eta \leq \ifrac{1}{2L}$ initialized at $w_1=0$ has, for any $\tail$ such that $0 < \tail/\rho(\epsilon)^2 \leq \ifrac{1}{\eta T}$,
\begin{align*}
    \E\brk[s]{ F(w_T) }
    =
    O\left( \frac{\rho(\tail)^2}{\eta T}+
    \frac{L \rho(\tail)^{4(1-\pow)} c^2\eta^{2\pow} T^{2\pow}}{n^{1+2\pow}}\right)
    .
\end{align*}
If in addition $f$ is $G$-Lipschitz, the output $w_T$ also satisfies
\begin{align*}
    \E\brk[s]{ F(w_T) }
    =
    O\left( \frac{\rho(\tail)^2}{\eta T}+
    \frac{\rho(\tail)^{2(1-\pow)}\eta^{\pow} T^{\pow} G\coef }{n}\right)
    .
\end{align*}
\end{theorem}

We remark that the condition on $\tail$ is not very restrictive; for example, $\tail = \ifrac{1}{\eta T}$ is always a valid choice and gives a nontrivial convergence bound. (In applications of the bound, we will sometimes use a better choice of $\tail$ that satisfies this condition.)

To obtain this result, we first bound the training error of GD in the smooth and convex regime,
under \cref{ass:tail}. The proof is standard and based on basic techniques in
convex optimization.

\begin{lemma} \label{lem:opt}  
Under \cref{ass:tail,ass:grad_norm_pow}, if for every $z$ $f(w,z)$ is
$L$-smooth, convex and positive, let $\{w_t\}_t$ be produced by the GD update rule
(\cref{gd_update_rule}) with $\eta \leq \ifrac{1}{2L}$ on $S$ and $w_1=0$. It
holds for any $\tail > 0$ that
\begin{align*}
    \widehat{F}(w_T)
    \leq
    \frac{1}{T}\sum_{t=1}^{T}\widehat{F}(w_t)
    \leq
    \frac{2\rho(\tail)^2}{\eta T}+ 2\tail
    .
\end{align*}
\end{lemma}

Now we turn to proving the generalization bound for gradient methods on self bounded objectives, as stated in \cref{main_thm_GD}. 
\begin{proof}[of \cref{main_thm_GD}]
First, we consider the case which in for any $z$, the loss function $f(w,z)$ is nonnegative, convex, $L$-smooth and $G$-Lipschitz with respect to $w$. Moreover, that $f$ satisfies  \cref{ass:tail,ass:grad_norm_pow}. 
Then, we know by \cref{lem:lip_gen} that 
\begin{align*}
    \E\brk[s]!{ F(w_t) }
    \leq
    \E\brk[s]!{ \widehat{F}(w_t) }
    + \E\left[\frac{2L}{n}\sum_{i=1}^n\|w_t-w_t^i\|\right]
    .
\end{align*}
By \cref{lem:opt} and \cref{exp stab power} we get,
\begin{align*}
    \E\brk[s]{ F(w_t) }
    \leq 
    \left(\frac{2\rho(\tail)^2}{\eta T}+ 2\tail\right)+\frac{2\eta T^\pow G\coef }{n}\left(\frac{2\rho(\tail)^2}{\eta}+ 2\tail T\right)^{1-\pow}
    .
\end{align*}
Finally, if $\epsilon 
\leq  \ifrac{\rho(\tail)^2}{\eta T}$ we get,
\begin{align*}
    \E\brk[s]{ F(w_t) }
    &\leq 
    \left(\frac{2\rho(\tail)^2}{\eta T}\right)+\frac{2\eta T^\pow G\coef }{n}\left(\frac{2\rho(\tail)^2}{\eta }\right)^{1-\pow}
    \\
    &=
    O\left( \frac{\rho(\tail)^2}{\eta T}+
    \frac{\rho(\tail)^{2(1-\pow)}\eta^{\pow} T^{\pow} G\coef }{n}\right)
    .
\end{align*}
If $f$ is non-Lipschitz, 
then by \cref{lem:non_lip_gen},
\begin{align*}
    \E\brk{ F(w_t) }
    \leq
    4\E[ \widehat{F}(w_t) ]
    + \E\left[\frac{3L}{n}\sum_{i=1}^n\|w_t-w_t^i\|^2\right]
    .
\end{align*}
By \cref{lem:opt} and \cref{stab_gd_nonlip} we get,
\begin{align*}
    \E\brk[s]{ F(w_t) }
    &\leq 
    4\left(\frac{2\rho(\tail)^2}{\eta T}+ 2\tail\right)+
    \frac{3L c^2\eta^2 T^{2\pow}}{n^{1+2\pow}} \brk3{\frac{2\rho(\tail)^2}{\eta }+ 2\tail T}^{2(1-\pow)}
    .
\end{align*}
Finally, if $\epsilon
\leq \ifrac{\rho(\tail)^2}{\eta T}$ we can simplify as follows:
\begin{align*}
    \E\brk[s]{ F(w_t) }
    &\leq 
    4\left(\frac{2\rho(\tail)^2}{\eta T}\right)+
    \frac{3L c^2\eta^2 T^2}{n^{1+2\pow}} \brk3{\frac{2\rho(\tail)^2}{\eta }}^{2(1-\pow)}
    \\
    &= 
    O\left( \frac{\rho(\tail)^2}{\eta T}+
    \frac{L \rho(\tail)^{4(1-\pow)} c^2\eta^{2\pow} T^{2\pow}}{n^{1+2\pow}}\right).
\end{align*}
\end{proof}

\subsection{Stochastic Gradient Descent}
 In this section we state and show a generalization bound for SGD on self-bounded losses.
\begin{theorem}\label{main_thm_sgd}
Suppose that for any $z$, the loss function $f(w,z)$ is nonnegative, convex, $G$-Lipschitz and $L$-smooth with respect to $w$. Further assume that $f$ satisfies  \cref{ass:tail,ass:grad_norm_pow}. 
Then the output of SGD with step size $\eta\leq\ifrac{1}{2L}$ initialized at $w_1=0$ has, for any $\tail$ such that $0 < \tail/\rho(\epsilon)^2 \leq \ifrac{1}{\eta T}$,
 \begin{align*}
    \E\brk[s]{ F(w_T) }
    =
    O\left( \frac{\rho(\tail)^2}{\eta T}+
    \frac{\rho(\tail)^{2(1-\pow)}\eta^{\pow} T^{\pow} G\coef }{n}\right)
    ,
 \end{align*}
when the expectation is on the randomness of the algorithm and on the data examples.
\end{theorem}

As in the case of GD, we use the stability bound to get a generalization error bound. 
In the same manner, we begin in bounding the optimization error of SGD, as detailed in the following lemma.

\begin{lemma}\label{lem:opt_sgd}
If for every $z$, $f(w,z)$ is $L$-smooth, convex ,positive and satisfy \cref{ass:tail}, let $\{w_t\}_t$ be produced by the SGD update rule (\cref{sgd_update_rule}) with $\eta \leq \ifrac{1}{2L}$ on $S$ and $w_1=0$. Then, for every $\tail>0$,
\begin{align*}
    \E\brk[s]!{ \widehat{F}\left(\overline{w}_T\right) }
    \leq 
    \E\brk[s]*{ \frac{1}{T}\sum_{t=1}^{T} \widehat{F}(w_t) }
    \leq
    \frac{\rho^2(\epsilon)}{2\eta T}+\tail
    ,
\end{align*}
when the expectation is on the randomness of the algorithm.
\end{lemma}
Then, we prove generalization in a similar way as in GD, except using \cref{lem:opt_sgd} and \cref{exp stab power sgd} instead of \cref{lem:opt} and \cref{exp stab power}. The full proof appears in \cref{gen_proofs}.

\section{Applications to separable linear classification}
\label{sec:examples}

As detailed in \cref{sec:setup}, the setting of linear classification with separable data (with margin $\gamma>0$) is captured by our general framework. 
In this section, we demonstrate how to apply our general generalization bounds for gradient methods in this setting with several popular choices of loss functions.
%
Most of our application are of binary classification, where $f_i(w) = \ell(w \cdot z_i)$ for every $i$, where $\ell : \R \to \R^+$ is a nonnegative loss function and $z_i=y_i\cdot x_i$ for the labeled examples $(x_i,y_i)$.
\subsection{Logistic loss}
We start with functions with exponential tails.
The first loss function that we consider is the logistic loss,
$
    \ell(y) = \log(1+e^{-y}).
$
This function is convex, $1$-Lipschitz, $1$-smooth and $(1,0)$-self-bounded (see \cref{lem:logloss_prop} in  \cref{apps_proofs}).
Moreover, by \cref{lem:tail_sep}, \cref{ass:tail} holds with $\rho(\tail) = \ipfrac{1}{\margin} \log\ipfrac{1}{\tail}$. 
For this function, we can choose $\tail = \ifrac{1}{T}$ and obtain the following generalization bound:

\begin{corollary} \label{cor:logloss_bound}
If $\ell$ is the logistic loss, then for gradient descent on $\widehat{F}$ with step size $\eta=\frac12$ and $w_1=0$:
\begin{align*}
     \E\brk[s]{ F(w_T) }
     =
     O\left( \frac{\log^2{T}}{\margin^2T} + \frac{\log^2{T}}{\margin^2n}\right)
    .
\end{align*}
\end{corollary}

The implied generalization bound over the zero-one accuracy of the model matches, up to $\log{T}$ factors, the bounds obtained by \cite{shamir2021gradient} for the normalized predictor. 
 
\subsection{Polynomially-tailed losses}

Now we turn to discuss loss functions with polynomially-decaying tails. 
On such functions, \cref{ass:grad_norm_pow} holds for $\delta=0$.
For concreteness, we will focus on the following loss function, but our arguments hold more generally for any 
Lipschitz loss with a similar tail decay rate:
\begin{equation} \label{poly_eq}
\ell(y)= \left\{
\begin{array}{ll}
    (1+y)^{-\polpow}\ & y \geq 0 ;
    \\
    1-\polpow y & y<0 ,
\end{array} 
\right.
\end{equation}
for $\polpow>0$, 
which is convex, $\polpow$-Lipschitz, $\polpow(\polpow+1)$-smooth, and $(\polpow,0)$-self-bounded (see \cref{lem:poly_prop} in \cref{apps_proofs}).
In addition, by \cref{lem:tail_sep}, \cref{ass:tail} holds for  $\rho(\tail)=\ipfrac{1}{\margin} \tail^{-\ifrac{1}{\polpow}}$. 
For a suitable choice of $\epsilon$
we can obtain the following generalization bound:

\begin{corollary} \label{gen polynomial}  
If $\ell$ is the the function defined in \cref{poly_eq}, the output of gradient descent on $\widehat{F}$ with step size
$\eta=\frac{1}{2L}=\frac{1}{2\polpow(\polpow+1)}$ and $w_1=0$ holds,
\begin{align*}
    \E\brk[s]{F(w_T)}
    =
    O\brk3{ \brk3{\frac{\polpow}{\margin}}^{\frac{2\polpow}{2+\polpow}} \brk3{\frac{1}{T^{\frac{\polpow}{2+\polpow}}}+\frac{T^{\frac{2}{\polpow+2}}}{n} } }
    .
\end{align*}
\end{corollary}

We see that as long as $\polpow \ll T$, as the degree of the polynomial gets higher,
gradient methods will optimize the test loss faster. 
Also note that this bound is optimized for $T=n$, for any degree $\alpha > 0$.
%
Finally, we remark that the rate we established in \cref{gen polynomial} is essentially the best one could hope for, as it matches the optimal (training) optimization rate of GD on polynomially-tailed functions.

\begin{lemma} \label{lem:poly_lower_bound}  
Let $\polpow\geq 1$. There exists a function $f:\R \to \R^+$ that is convex, $\polpow$-Lipschitz,
$\polpow(\polpow+1)$-smooth and $(\coef$,0)-self-bounding over $\R^d$ with
$\rho(\tail)=\tail^{-\ifrac{1}{\polpow}}$ and for every $t\geq 1$, the iterate $w_t$ of
gradient descent with $\eta=\ifrac{1}{\polpow(\polpow+1)}$ and $w_1=0$ has
$f(w_t)=\Omega(t^{-\frac{\polpow}{\polpow+2}})$.
\end{lemma}

\subsection{Cross Entropy Loss with Softmax Output}

Now, we discuss the cross entropy loss with softmax output, 
\begin{equation}
\label{cross_eq}
    f_i(\{w\}_{i=1}^k)
    = 
    -\sum_{i=1}^K\log\left(\frac{\exp(w_{y_i} \cdot x_i)}{\sum_{j=1}^k \exp(w_{k} \cdot x_i)}\right)
    ,
\end{equation}
where $\{(x_i,y_i)\}_{i=1}^n$ are the labeled examples.
This function is a loss function is the generalization of the logistic loss for the setting of
multi class classification
by linear predictors. We use the fact that our bounds are in the general convex
optimization setting to get the general case of $K$ classes easily. We discuss
the case of separable data with margin $\margin$: there exists a unit vector
$w^*$ such that for all $z_i$ and $j\neq y_i$, $y_iw^*\cdot\left(e_{y_{i}}\otimes
I_d-e_j\otimes I_d\right)x_i\geq \margin$, when $w^*$ is a concatenation of $K$ vectors $\{w_i^*\}_{i=1}^K$ and $\otimes$ is the Kronecker product. The cross entropy loss with softmax output satisfies our assumptions (see \cref{lem:crossentropy_prop} in \cref{apps_proofs}). Specifically, it satisfies \cref{ass:grad_norm_pow} with $\pow=0,\coef=2$ and \cref{ass:tail} with $\rho(\tail) = \frac{1}{\margin} \log(\frac{K}{\tail})$. We can conclude for $\tail=\ifrac{1}{T}$:

\begin{corollary}
If $f_i$ is the function defined by \cref{cross_eq}, the output of gradient descent on $\widehat{F}$ with step size $\eta=\frac{1}{2L}=\frac{1}{4}$ and $w_1=0$
holds, 
\begin{align*}
     \E\left[F(w_T)\right]= O\left( \frac{\log^2(KT)}{\margin^2T}+
     \frac{ \log^2(KT)}{\margin^2n}\right).
    \end{align*}
\end{corollary}

\subsection{Sub-exponential tail}

Now, we discuss the case which the tail is super-polynomial and sub-exponential,
\begin{equation}
\label{subexp_eq}
\ell(y)= \left\{
\begin{array}{ll}
      e^{-(1+y)^{\polpow}}\ & y \geq 0 \\
        \frac{1}{e}(1-\polpow y) & y<0\\
\end{array} ,
\right.
\end{equation}
where $0<\polpow\leq 1$. This function is convex, $\polpow$-Lipschitz, $\polpow$-smooth and $(\polpow,0)$-self-bounded (see \cref{lem:subexp_prop} in \cref{apps_proofs}).
In addition, by \cref{lem:tail_sep}, \cref{ass:tail} holds for 
$\rho(\tail) = \frac{1}{\margin} \log^{\ifrac{1}{\polpow}}(\frac{1}{\tail})$.
As a result, we can deduce for $\tail=\ifrac{\polpow}{T}$ the following generalization bound,
\begin{corollary}
If $\ell$ is the function defined in \cref{subexp_eq}, then the output of gradient descent on $\widehat{F}$ with step size $\eta=\frac{1}{2\polpow}$ and $w_1=0$ holds,
    \begin{align*}
     \E\left[F(w_T)\right]= O\left(\frac{\polpow \log^{\frac{2}{\polpow}}(T)}{\margin^2 T}+\frac{\polpow^2\log^{\frac{2}{\polpow}}(T)}{\margin^2 n}\right)
    .
 \end{align*}
 \end{corollary}
 \subsection{Super-exponential tail}
 Now we discuss function that goes to $0$ in a rate which is faster than exponential rate. 
 We show examples for functions which holds \cref{ass:grad_norm_pow} for arbitrary small $\pow>0$, and get a generalization bound for them.
 First, we consider the function
\begin{equation}
\label{supexp_eq}
\ell(y)= \left\{
\begin{array}{ll}
      e^{-(1+y)^{\polpow}}\ & y \geq 0 \\
        \frac{1}{e}(1-\polpow y) & y<0\\
\end{array} 
\right.,
\end{equation}
for $\polpow \geq 1$.
This function satisfy our assumptions. Specifically, \cref{ass:grad_norm_pow} holds for any $\pow>0$ with a match coefficient $\coef$. Further, $\ell$ is convex, $\polpow$-Lipschitz, $\polpow^2$-Smooth, and for every $\pow>0$, it is $(\frac{\polpow}{e\pow},\pow)$-self-bounded (see \cref{lem:superexp_prop} in \cref{apps_proofs}).
Moreover,  by \cref{lem:tail_sep}, \cref{ass:tail} holds for  $\rho(\tail) = \frac{1}{\margin} \log^{\ifrac{1}{\polpow}}(\frac{1}{\tail})$.
For $\tail=\ifrac{\polpow}{T}$ and $\pow={\ifrac{1}{\log T}}$ we get the following risk bound:
\begin{corollary}
\label{lem:supexp_gen}
If $\ell$ is the function defined in \cref{supexp_eq}, the output of gradient descent on $\widehat{F}$ with step size $\eta=\frac{1}{2\polpow^2}$ and $w_1=0$ 
    \begin{align*}
     \E\left[F(w_T)\right]= O\left(\frac{\polpow^2 \log^{\frac{2}{\polpow}}(T)}{\margin^2 T}+\frac{\polpow^{3} \log^{\frac{2+\polpow}{\polpow}}(T)}{\margin^{2} n}\right)
    .
 \end{align*}
\end{corollary}

\subsection{Probit loss}

Finally, we consider the classical probit loss, defined as
$
    \ell(y)
    =
    -\log(\tfrac12 (1-\operatorname{erf}(y)))
    ,
$
when $\operatorname{erf}(y)=\frac{2}{\sqrt{\pi}}\int_0^ye^{-t^2}dt$; this loss is convex and smooth, but not Lipschitz, thus demonstrates our bound for gradient descent in the smooth but non-Lipschitz regime. In addition $\ell$ satisfies \cref{ass:tail} for every $\tail\leq \frac{1}{2}$ and $\rho(\tail)=(\ifrac{1}{\margin}) \sqrt{\log(\ifrac{1}{\tail})}$, and for every $\pow\leq \frac{1}{2}$, $\ell(y)$ is ($\frac{8}{e\pow},\pow$)-self bounded. (see \cref{lem:probit_prop} in \cref{apps_proofs}).
By all of the above, we can conclude for $\tail=\ifrac{1}{T}$ and $\pow=\ifrac{1}{(2\log T)}$:
\begin{corollary}
\label{probit gen}
If $\ell$ is the probit loss, the output of gradient descent on $\widehat{F}$ with $\eta=\frac{1}{4}$ and $w_1=0$ holds
    \begin{align*}
     \E\left[F(w_T)\right]
    = O\left(\frac{\log{T}}{T\margin^2}
    +  \frac{\log^4{T}}{n\margin^{4}}\right)
    .
\end{align*}
\end{corollary}



\input{acknowledgement_deutch}

\bibliography{main}
\newpage
\appendix
 \section{Proofs of \cref{sec:setup}}
 \label{setting_proof}

 \begin{proof}[of \cref{lem:grad_power_equiv}]
 \begin{enumerate}[label=(\roman*)]
     \item We know that 
\begin{align*}
    c\geq \|\nabla \log f(w)\|= \frac{\|\nabla f(w)\|}{ f(w)}.
\end{align*}
This is equivalent to 
\begin{align*}
    \|\nabla f_i(w)\|\leq c f_i (w).
\end{align*}
\item We know that 
\begin{align*}
    c\delta\geq \|\nabla f^\pow(w)\|= \|\nabla f(w)\|\pow f^{\pow-1}(w).
\end{align*}
This is equivalent to 
\begin{align*}
    \|\nabla f(w)\|\leq cf^{1-\pow}(w).
\end{align*}
 \end{enumerate}
 \end{proof}

\begin{proof}[of \cref{lem:tail_sep}]
Let $\tail$. First, we show that $\ell$ is invertible. $\ell$ is injective by the fact that it is strictly monotonically decreasing. Therefore, it is sufficient to prove that there exists $x$ with $\ell(x)=\epsilon$.
We will show that $\lim_{x \to -\infty} \ell(x) = \infty$ and this will be implied by combining it with the fact that $\lim_{x \to \infty} \ell(x) = 0$.
Now, if $f'(0)=0$, by convexity, for every $a\geq0$ $f'(a)\geq0$, which is a contradiction to the strict monotnoicity of $f$. We obtain that $f'(0)<0$.
Then, by convexity,
\begin{equation*}
    \lim_{x \to -\infty} \ell(x)\geq  \lim_{x \to -\infty}(f(0)+f'(0)x)=f(0)+f'(0)\lim_{x \to -\infty}x=\infty.
\end{equation*}
Now, by separability, there exists a unit vector $w^*$ such that for every $z$, $w^*\cdot z > \margin$.
$\ell$ is monotonic decreasing. Then for $w=\frac{\ell^{-1}(\tail)}{\margin}\cdot w^*$,
\begin{align*}
    f(w,z)\leq \ell\left(\frac{\ell^{-1}(\tail)}{\margin}\cdot w^* \cdot z\right) \leq \ell\left(\frac{\ell^{-1}(\tail)}{\margin}\cdot \margin\right)=\tail.
\end{align*}
\end{proof}

\section{Proofs of \cref{sec:stability}}
\label{stab_proofs}
\begin{proof}[of \cref{lem:lip_gen}]
Denote by $z$ an instance independent from the samples $S,S_i$. Moreover, denote by $\widehat w$ the output of $A$ on $S$, and by $\widehat w_i$ the output of A on $S_i$.
We know that for every sample $S$ and index $i$, the random variables $\widehat w_i$,$z_i$ are independent. 
As a result,
\begin{equation*}
    \E_{S,A} f(\widehat w_i,z_i) 
    = \E_{S_i,A} F(\widehat w_i) 
    = \E_{S_i,A,z} f(\widehat w_i,z)=\E_{S,A,z}f(\widehat w_i,z)
    .
\end{equation*}
We get,
\begin{align*}
\Esalg\brk[s]{ F(\widehat w)-\widehat{F}(\widehat w) }
 &=
 \frac{1}{n}\sum_{i=1}^n \brk[s]{ \E_{S,A,z}f(\widehat w,z) - \Esalg f(\widehat w,z_i) }
 \\
 &=
 \frac{1}{n}\sum_{i=1}^n \E_{S,A,z}\brk[s]{ f(\widehat w,z)-f(\widehat w_i,z) } 
    + \frac{1}{n}\sum_{i=1}^n \Esalg\brk[s]{ f(\widehat w_i,z_i)-f(\widehat w,z_i) }
 \\
 &\leq 
 \Esalg\brk[s]*{ \frac{2G}{n}\sum_{i=1}^n\|\widehat w- \widehat w_i\| }
 .
\end{align*}\end{proof}
For proving \cref{lem:non_lip_gen}, we rely on the following standard lemma about smooth functions (proof can be found in, e.g., \citealp{nesterov2003introductory}).
\begin{lemma} \label{lem:2L_serbro}  
For a non-negative and $L$-smooth $f:\R^d \to \R$, it holds that $\|\nabla
f(w)\|^2 \leq 2L f(w)$ for all $w\in \R^d$.
\end{lemma}

\begin{proof}[of \cref{lem:non_lip_gen}]
Denote by $\widehat w$ the output of $A$ on $S$, and by $\widehat w_i$ the output of A on $S_i$.
Using the $L$-smoothness of each $f(\cdot,z)$, and the inequality $u \dotp v \leq \tfrac{1}{2\gamma} \norm{u}^2 + \tfrac{\gamma}{2} \norm{v}^2$ (that holds for any $\gamma>0$) we have for every $z$ and for all $i$ that
\begin{align*} 
    f(\widehat w_i,z) - f(\widehat w,z)
    &\leq
    \nabla f(\widehat w,z) \dotp (\widehat w_i-\widehat w) + \frac{L}{2} \norm{\widehat w_i-\widehat w}^2
    \\
    &\leq
    \frac{1}{2L} \norm{\nabla f(\widehat w,z)}^2 + \frac{L}{2} \norm{\widehat w_i-\widehat w}^2 + \frac{L}{2} \norm{\widehat w_i-\widehat w}^2
    \\
    &=
    \frac{1}{2L} \norm{\nabla f(\widehat w,z)}^2 + L \norm{\widehat w_i-\widehat w}^2
    \\
    &\leq 
    f(\widehat w,z) + L \norm{\widehat w_i-\widehat w}^2
    ,
\end{align*}
where in the final inequality we used $L$-smoothness again (\cref{lem:2L_serbro}). Thus,
\begin{align} \label{eq:stab-sm1}
    f(\widehat w_i,z) 
    \leq
    2 f(\widehat w,z) + L \norm{\widehat w_i-\widehat w}^2
    .
\end{align}
Similarly we can show, for any $z$ and for all $i$ ,
\begin{align} \label{eq:stab-sm2}
    f(\widehat w,z)
    \leq
    2 f(\widehat w_i,z) + L \norm{\widehat w-\widehat w_i}^2
    .
\end{align}
We know that for every sample $S$ and index $i$, the random variables $\widehat w_i$
and $z_i$ are independent. As a result,
\begin{equation*}
    \E_{S,A} f(\widehat w_i,z_i) 
    = \E_{S_i,A} F(\widehat w_i) 
    = \E_{S_i,A,z} f(\widehat w_i,z)=\E_{S,A,z}f(\widehat w_i,z)
    .
\end{equation*}
Then,
\begin{align*}
    \Esalg\brk[s]{ F(\widehat w) }
    &= 
    \E_{S,A,z} \brk[s]{ f(\widehat w,z) }
    \\
    &\leq 
    \frac{1}{n} \sum_{i=1}^n \E_{S,A,z}\brk[s]!{ 2 f(\widehat w_i,z) 
    + L \norm{\widehat w_i-\widehat w}^2 }
    \tag{\cref{eq:stab-sm2}}
    \\
    &=
    \frac{2}{n}\sum_{i=1}^n \Esalg\brk[s]!{ f(\widehat w_i,z_i) } 
    + \frac{L}{n} \sum_{i=1}^n \Esalg\brk[s]!{ \norm{\widehat w_i-\widehat w}^2 }
    \\
    &\leq
    \frac{2}{n} \sum_{i=1}^n \Esalg\brk[s]!{ 2f(\widehat w,z_i) 
    + L \norm{\widehat w_i-\widehat w}^2 } 
    + \frac{L}{n} \sum_{i=1}^n \Esalg\brk[s]!{ \norm{\widehat w_i-\widehat w}^2 }
    \tag{\cref{eq:stab-sm1}}
    \\
    &= 
    4 \Esalg\brk[s]{ \widehat{F}(\widehat w) }
    + \frac{3L}{n} \sum_{i=1}^n \Esalg\brk[s]!{ \norm{\widehat w_i-\widehat w}^2 }
    .
\end{align*}
\end{proof}

\section{Proofs of \cref{sec:bounds}}
\label{gen_proofs}

\begin{proof}[of \cref{lem:opt}]
First, by smoothness, for every $t$ and $\eta \leq \frac{1}{L}$
\begin{align*}
    \widehat{F}(w_{t+1})&\leq \widehat{F}(w_t) + \nabla \widehat{F}(w_t)(w_{t+1}-w_t) +\frac{L}{2} \|w_{t+1}-w_t\|^2
    \\&=
    \widehat{F}(w_t) - \eta  \|\nabla \widehat{F}(w_t)\|^2  + \frac{\eta^2 L}{2} \|\nabla \widehat{F}(w_t)\|^2
    \\&\leq
    \widehat{F}(w_t) - \frac{\eta}{2}\|\nabla \widehat{F}(w_t)\|^2
    \\&\leq
    \widehat{F}(w_t).
\end{align*}
We conclude,
\begin{equation}
\label{GD_mono}
    \widehat{F}(w_T)\leq \frac{1}{T} \sum_{t=1}^{T}\widehat{F}(w_t).
\end{equation}
From standard regret bounds for gradient updates, we obtain that for every $w$,
\begin{align*}
    \frac{1}{T} \sum_{t=1}^{T} \brk{ \widehat{F}(w_t) - \widehat{F}(w) }
    &\leq
    \frac{\norm{w_1-w}^2}{2\eta T} + \frac{\eta}{2T} \sum_{t=1}^{T} \norm{\nabla \widehat{F}(w_t)}^2
    .
\end{align*}
By \cref{lem:2L_serbro},
\begin{align*}
    \frac{1}{T} \sum_{t=1}^{T} \brk{ \widehat{F}(w_t) - \widehat{F}(w) }
    &\leq
    \frac{\norm{w}^2}{2\eta T} + \frac{\eta L}{T} \sum_{t=1}^{T} \widehat{F}(w_t)
    .
\end{align*}
Using $\eta L \leq \frac{1}{2}$ and rearranging gives
\begin{align*}
    \frac{1}{T} \sum_{t=1}^{T} \brk{ \widehat{F}(w_t) - 2\widehat{F}(w) }
    &\leq
    \frac{\norm{w}^2}{\eta T}
    .
\end{align*}
By \cref{ass:tail}, for all $\tail$, there exists $w^*$ such that for all $i$, $f_i(w^*)\leq \tail$, $\|w^*\|\leq \rho(\tail)$. For $w=w^*$ we get by \cref{GD_mono}
 \begin{align*}
    \widehat{F}(w_T)\leq \frac{1}{T}\sum_{t=1}^{T} \widehat{F}(w_t)
    \leq 
    \frac{\norm{w^*}^2}{\eta T} + 2\widehat{F}(w^*)
    \leq
    \frac{2\rho(\tail)^2}{\eta T} + 2\tail
    .
 \end{align*}
\end{proof}
\begin{proof}[of \cref{lem:opt_sgd}]
For every $w$, iteration $j$ and possible $i_j$, by \cref{lem:2L_serbro} and convexity,
\begin{align*}
    \|w_{j+1}-w\|^2&\leq \|w_{j}-w\|^2 -2\eta \langle \nabla f(w_j,z_{i_{j}})(w_j-w)\rangle +\eta^2\|\nabla f(w_j,z_{i_{j}})\|^2
    \\&\leq 
    \|w_{j}-w\|^2 +2\eta  f(w,z_{i_{j}})-2\eta f(w_j,z_{i_{j}}) +2\eta^2L f(w_j,z_{i_{j}})
    \\&\leq 
    \|w_{j}-w\|^2 +2\eta  f(w,z_{i_{j}})- 2\eta  f(w_j,z_{i_{j}})
    .
\end{align*}
Taking expectation on the order of the samples ($w_j,i_j$ are independent),
\begin{align*}
    \E\left[\frac{1}{n}\sum_{i=1}^n f(w_j,z_i)\right]-\frac{1}{n}\sum_{i=1}^n f(w,z_i)
    \leq \frac{1}{2\eta} \left(\E\|w_{j}-w\|^2-\E \|w_{j+1}-w\|^2\right) 
    .
\end{align*}
Taking average on $j=1...{T}$ and using Jensen Inequality,
\begin{align*}
    \E\left[\widehat{F}\left(\overline{w}_t\right)\right]\leq \E\left[\frac{1}{T}\sum_{t=1}^{T} \widehat{F}(w_t)\right]\leq \widehat{F}(w) + \frac{\|w_{1}-w\|^2}{2\eta T}=\widehat{F}(w) + \frac{\|w\|^2}{2\eta T}
    .
\end{align*}
For $w=w^*$ we get the lemma.
\end{proof}

\begin{proof}[of \cref{main_thm_sgd}]
We know by \cref{lem:lip_gen},
\begin{align*}
   \E\left[F(\overline{w}_t)\right]\leq
    \E\left[\widehat{F}(\overline{w}_t) \right]
    + \E\left[\frac{2L}{n}\sum_{i=1}^n\|\overline{w}_t-\overline{w}_t^i\|\right]
    .
\end{align*}
By \cref{lem:opt_sgd} we get, 
\begin{align*}
    \E\brk[s]!{ \widehat{F}\left(\overline{w}_T\right) }
    \leq
    \frac{1}{T}\sum_{t=1}^{T}
\widehat{F}(w_t)
    \leq \frac{\rho^2(\epsilon)}{2\eta T}+\tail 
    ,
\end{align*}
Therefore, by \cref{exp stab power sgd},
\begin{align*}
    \E\left[F(\overline{w}_t)\right]\leq \left(\frac{\rho(\tail)^2}{2\eta T}+ \tail\right)+\frac{2\eta G\coef T^\pow }{n}\left(\frac{\rho(\tail)^2}{\eta}+ \tail T\right)^{1-\pow}
    .
\end{align*}
Finally, if $\epsilon
\leq \ifrac{\rho(\tail)^2}{\eta T}$, we get
\begin{align*}
    \E\left[F(\overline{w}_t)\right]&\leq \left(\frac{\rho(\tail)^2}{2\eta T}\right)+\frac{2\eta G\coef T^\pow }{n}\left(\frac{\rho(\tail)^2}{\eta}\right)^{1-\pow}
    \\
    &=
    O\left( \frac{\rho(\tail)^2}{\eta T}+
    \frac{\rho(\tail)^{2(1-\pow)}\eta^{\pow} T^{\pow} G\coef }{n}\right)
    .
\end{align*}
\end{proof}

 \section{Proofs of \cref{sec:examples}}
 \label{apps_proofs}
First, we introduce a lemma which will help us to build functions which satisfy our
assumptions, given a monotonic decreasing tail function $l:[a,\infty)\to \R^+$.
Using this lemma we will construct several useful loss functions and demonstrate
our bounds on these functions.

\begin{lemma}
\label{completion lemma}
Let $f$ be a positive function such that $\lim_{x\to\infty}f(x)=0$ and monotonic decreasing. For every $x$ in $\R^d$ such that $\|x\|\leq 1$, we define $g_x(w)=x\cdot w$. We define the following function
\[ 
h_x(w)= \left\{
\begin{array}{ll}
      f(g_x(w)) & g_x(w) \geq a \\
    f(a)\left(1+\frac{f' (a)}{f(a)}(g_x(w)-a)\right)& g_x(w)<a\\
\end{array} 
\right. 
.\]

\begin{enumerate}[label=(\roman*)]
    \item If $f$ is
 G-Lipschitz, convex and L-smooth on $[a,\infty]$, 
then, for any such $x$, the function $h_x(w)$ is $G$-Lipschitz, convex and $L$-smooth in $\R^d$.
\item 
If $f$ is
 ($\coef$,$\pow$)-self-bounded on $[a,\infty]$, 
then, for any such $x$, the function $h_x(w)$ is also ($\coef$,$\pow$)-self-bounded in $\R^d$.
\end{enumerate} 
\end{lemma}

\begin{proof}[of \cref{completion lemma}]
\begin{enumerate}[label=(\roman*)]
\item
First, for every $w$ with $g_x(w)< a$, $\nabla h_x(w)=f'(a)\cdot x$.
\begin{itemize}
    \item Convexity- If $a\leq g_x(z)\leq g_x(y)$,
    \begin{align*}
        \left(\nabla h_x(z)-\nabla h_x(y)\right)\cdot\left(z-y\right) 
        &= 
        \left(f'(g_x(z))\cdot x-f'(g_x(y))\cdot x\right)\cdot \left(z-y\right)
        \\
        &=
        (f'(g_x(z))-f'(g_x(y)))\cdot x\cdot (z-y) 
        \\
        &= 
        (f'(g_x(z))-f'(g_x(y))) (g_x(z)-g_x(y)).
    \end{align*}
    $f'$ is monotonic increasing in $[a,\infty)$, thus, $f'(g(x(z)) \leq f'(g_x(y))$. Moreover, $g_x(z)\leq g_x(y)$. Then
    \begin{equation*}
        (\nabla h_x(z)-\nabla h_x(y))(z-y) \geq 0.
    \end{equation*}

    If $g_x(z)\leq g_x(y)\leq a$,  $f$ is linear in $w$ and thus convex. 
    
    If $g_x(z)\leq a\leq g_x(y)$, then
    \begin{align*}
        (\nabla h_x(z)-\nabla h_x(y))(z-y)
        &= 
        (f'(a)\cdot x-f'(g_x(y))\cdot x) (z-y) 
        \\
        &=
        (f'(a)-f'(g_x(y)))x(z-y)
        \\
        &= 
        (f'(a)-f'(g_x(y))) (g_x(z)-g_x(y)).
    \end{align*}
    Now, $f'$ is monotonic increasing in $[a,\infty)$, thus, $f'(a) \leq f'(g_x(y))$. Moreover, $g_x(z)\leq g_x(y)$.
    We got,
    \begin{equation*}
        (\nabla h_x(z)-\nabla h_x(y))(z-y) \geq 0.
    \end{equation*}
    Then, $h_x$ is convex.
    \item Lipschitzness-
    If $g_x(w)\leq a$, \begin{equation*}
        \|h_x(w)\|=|f'(a)|\|x\|\leq f'(a).
    \end{equation*} Then, by the Lipschitzness of $f$ in $[a,\infty)$, $\|h_x(w)
    \|\leq G$.
    
    If $g_x(w)\geq a$, $|f'(g_x(w))|\leq G$ by the Lipschitzness of $f$ in $[a,\infty)$.  Then,
    \begin{equation*}
        \|h_x(w)\|=|f'(g_x(w))|\|x\|\leq G.
    \end{equation*}
    \item Smoothness-  If $g_x(z)\leq a\leq g_x(y)$, then, By the L-smoothness of $f$ in $[a,\infty)$,
    \begin{align*}
        \|\nabla h_x(y) -\nabla h_x(z)\| 
        &= 
        \|f'(g_x(y))\cdot x- f'(a)\cdot x\| 
        \leq
        |f'(g_x(y))-f'(a)|
        \leq
        L|g_x(y)-a|
        \\
        &\leq L|g_x(y)-g_x(z)|
        =
        L|x\cdot(y-z)| \leq L\|x\|\|y-z\|\leq L\|y-z\|.
    \end{align*}
    If $g_x(z)\leq g_x(y)\leq a$, then
    \begin{equation*}
         \|\nabla h_x(y) -\nabla h_x(z)\| =\|f'(a)\cdot x- f'(a)\cdot x\| = 0 \leq L\|y-z\|.
    \end{equation*}
     If $a\leq g_x(z)\leq g_x(y)$, then, by the 1-smoothness of $f$ in $[a,\infty)$
     \begin{align*}
        \|\nabla h_x(y) -\nabla h_x(z)\| 
        &= 
        \|f'(g_x(y))\cdot x- f'(g_x(z))\cdot x\| 
        \leq
        |f'(g_x(y))-f'(g_x(z))|
        \\&\leq L|g_x(y)-g_x(z)|
        =
        L|x\cdot (y-z)| \leq L\|x\|\|y-z\|\leq L\|y-z\|.
    \end{align*}
\end{itemize}

\item ($\coef$,$\pow$)-self-bounding- If $g_x(w)\geq a$ ,$h_x$ is ($\coef$,$\pow$)-self-bounded.
    Specifically, $\|f'(a)\|\leq cf^{1-\pow}(a)$.
    If $g_x(w)\leq a$,
     by fact that $f'(a)\leq 0$, we notice that 
    \begin{equation*}
        1+\frac{f'(a)}{f(a)}(g_x(w)-a)\geq 1.
    \end{equation*}
     \begin{equation*}
        \left(1+\frac{f'(a)}{f(a)}(g_x(w)-a)\right)^{1-\pow}\geq 1.
    \end{equation*}
    As a result,
    We get,
    \begin{equation*}
        \|\nabla h_x(w)\| = |f'(a)|\cdot \|x\|  \leq |f'(a)| \leq  cf^{1-\pow}(a) \left(1+\frac{f'(a)}{f(a)}(g_x(w)-a)\right)^{1-\pow}= c h_x(w)^{1-\pow}.
    \end{equation*}
\end{enumerate}
\end{proof}

\begin{lemma}
\label{lem:logloss_prop}
The logistic loss is convex, $1$-Lipschitz and $1$-smooth, $(1,0)$-self-bounded.
\end{lemma} 
\begin{proof}[of \cref{lem:logloss_prop}]
For every $y$,
\begin{align*}
    | \ell'(y)|=\frac{e^{-y}}{1+e^{-y}}\leq 1.
\end{align*}
\begin{align*}
    \ell''(y)=\frac{e^{-y}(1+e^{-y})-e^{-2y}}{(1+e^{-y})^2}=\frac{1}{(1+e^{-y})^2}.
\end{align*}
We got that $\ell$ is $1$-Lipschitz,$1$-smooth and convex.
Moreover, by the fact that for every $z>0$, $\frac{x}{1+x}\leq \log(1+x)$ 
\begin{align*}
   |\ell'(x)|= \frac{e^{-y}}{1+e^{-y}}\leq \log(1+e^{-y})=\ell(x).
\end{align*}
\end{proof}

\begin{lemma} \label{lem:crossentropy_prop}
The function defined in \cref{cross_eq} is  convex, $2$-smooth, $2$-Lipschitz, $(2,0)$-self-bounding, satisfy \cref{ass:grad_norm_pow} with $\pow=0,\coef=2$ and satisfy \cref{ass:tail} with $\rho(\tail) = \frac{1}{\margin} \log(\frac{K}{\tail})$.
\end{lemma}
 \begin{proof}[of \cref{lem:crossentropy_prop}]
 For convenience we use the notation of \cite{soudry2018implicit},
 \begin{equation*}
   f_i(w)=\log\left(\sum_{j=1}^K \exp(-w\left(e_{y_{i}}\otimes I_d-e_j\otimes I_d)z_i\right)\right).
 \end{equation*}
 when $w$ is a concatenation of $K$ vectors $\{w_i\}_{i=1}^K$ and $\otimes$ is the Kronecker product. 
 First, we show convexity, for every $a,b$ and $\lambda \in [0,1]$, by holder inequality for $p=\frac{1}{\lambda}$, $q=\frac{1}{1-\lambda}$,
 \begin{align*}
    &\lambda f_i(a) + (1-\lambda)f_i(b)
    \\&=
    \lambda \log\left(\sum_{j=1}^K \exp(- a\left(e_{y_{i}}\otimes I_d-e_j\otimes I_d)z_i\right)\right)+(1-\lambda)\log\left(\sum_{j=1}^Kexp(-b\left( e_{y_{i}}\otimes I_d-e_j\otimes I_d)z_i\right)\right)
     \\&=
     \log\left(\left(\sum_{j=1}^K \exp(- a\left(e_{y_{i}}\otimes I_d-e_j\otimes I_d)z_i\right)\right)^\lambda\right)+\log\left(\left(\sum_{j=1}^Kexp(-b\left( e_{y_{i}}\otimes I_d-e_j\otimes I_d)z_i\right)\right)^{1-\lambda}\right)
     \\ &=
     \log\left(\left(\sum_{j=1}^K \exp(- a\left(e_{y_{i}}\otimes I_d-e_j\otimes I_d)z_i\right)\right)^\lambda\left(\sum_{j=1}^Kexp(-b\left( e_{y_{i}}\otimes I_d-e_j\otimes I_d)z_i\right)\right)^{1-\lambda}\right)
     \\&\geq
     \log\left(\sum_{j=1}^K \exp(-\lambda a\left(e_{y_{i}}\otimes I_d-e_j\otimes I_d)z_i\right)exp(-(1-\lambda)b\left( e_{y_{i}}\otimes I_d-e_j\otimes I_d)z_i\right)\right)
     \\&=
     \log\left(\sum_{j=1}^K \exp(-(\lambda a+(1-\lambda)b)\left(e_{y_{i}}\otimes I_d-e_j\otimes I_d)z_i\right)\right)
     \\&=
     f_i(\lambda a+(1-\lambda)b).
      \end{align*} 
      In addition,
 \begin{align*}
     \|\nabla f_i(w)\|&=\frac{\sum_{j=1}^K\|\left(e_{y_{i}}\otimes I_d-e_j\otimes I_d)z_i\right)\|\exp(-w\left(e_{y_{i}}\otimes I_d-e_j\otimes I_d)z_i\right)}{\sum_{j=1}^K \exp(-w\left(e_{y_{i}}\otimes I_d-e_j\otimes I_d)z_i\right)}
     \\&=
     \frac{\sum_{j\neq y_i}\|\left(e_{y_{i}}\otimes I_d-e_j\otimes I_d)z_i\right)\|\exp(-w\left(e_{y_{i}}\otimes I_d-e_j\otimes I_d)z_i\right)}{1+\sum_{j\neq y_i}\exp(-w\left(e_{y_{i}}\otimes I_d-e_j\otimes I_d)z_i\right)}
     \\&\leq
     \frac{2\sum_{j\neq y_i}\exp(-w\left(e_{y_{i}}\otimes I_d-e_j\otimes I_d)z_i\right)}{1+\sum_{j\neq y_i}\exp(-w\left(e_{y_{i}}\otimes I_d-e_j\otimes I_d)z_i\right)}
     \\&\leq 2\log\left(1+\sum_{j\neq y_i}\exp(-w\left(e_{y_{i}}\otimes I_d-e_j\otimes I_d)z_i\right)\right)=2f_i(w).
 \end{align*}
  We conclude,
 \begin{align*}
     \|\nabla f_i(w)\|\leq 2.
 \end{align*}
 \begin{align*}
     \|\nabla f_i(w)\|\leq 2f_i(w).
 \end{align*}
 Moreover, 
 \begin{align*}
    & \|\nabla^2 f_i(w)\|_{op}\\&=\frac{\sum_{j=1}^K\|\left(e_{y_{i}}\otimes I_d-e_j\otimes I_d)z_ix_j^T\right)\|_{op}\exp(-w\left(e_{y_{i}}\otimes I_d-e_j\otimes I_d)z_i\right)\sum_{j=1}^K\exp(-w\left(e_{y_{i}}\otimes I_d-e_j\otimes I_d)z_i\right)}{\left(\sum_{j=1}^K\exp(-w\left(e_{y_{i}}\otimes I_d-e_j\otimes I_d)z_i\right)\right)^2}
     \\&-\frac{\|\left(\sum_{j=1}^K\left(e_{y_{i}}\otimes I_d-e_j\otimes I_d)z_i\right)\exp(-w\left(e_{y_{i}}\otimes I_d-e_j\otimes I_d)z_i\right)\right)\|^2}{\left(\sum_{j=1}^K\exp(-w\left(e_{y_{i}}\otimes I_d-e_j\otimes I_d)z_i\right)\right)^2}
     \\&\leq 
     \frac{2\sum_{j\neq i}\exp(-w\left(e_{y_{i}}\otimes I_d-e_j\otimes I_d)z_i\right)}{1+\sum_{j\neq i}\exp(-w\left(e_{y_{i}}\otimes I_d-e_j\otimes I_d)z_i\right)}+0
     \\&\leq 
     2
 \end{align*}
  In addition, for $w=\frac{\log(\frac{k-1}{\tail})}{\margin}w^*$:
  \begin{align*}
      f_i(w)&=\log\left(\sum_{j=1}^K\exp(-w\left(e_{y_{i}}\otimes I_d-e_j\otimes I_d)z_i\right)\right)
      \\&=
      \log\left(1+\sum_{j\neq y_i}\exp(-w\left(e_{y_{i}}\otimes I_d-e_j\otimes I_d)z_i\right)\right)
      \\&\leq
        \sum_{j\neq y_i}\exp(-w\left(e_{y_{i}}\otimes I_d-e_j\otimes I_d)z_i\right)
      \\&\leq
      \sum_{j\neq y_i}\exp(-\margin \frac{\log(\frac{k-1}{\tail})}{\margin})
      \\&\leq
      \sum_{j\neq y_i}\frac{\tail}{k-1}      \\&\leq\tail.
  \end{align*}
 \end{proof}
\begin{lemma} \label{lem:poly_prop}
The function defined in \cref{poly_eq} is convex, $\polpow$-Lipschitz, $(\polpow^2+\polpow)$-smooth, and $(\polpow,0)$-self-bounded.
\end{lemma}
\begin{proof}[of \cref{lem:poly_prop}]
We define $\ell(y)=(1+y)^{-\polpow}$.
For every $y \geq 0$:
\begin{equation*}
    |\ell'(y)|=\polpow(1+y)^{-\polpow-1}\leq \polpow,
\end{equation*}
\begin{equation*}
    \ell''(y)=\polpow(\polpow +1) (1+y)^{-\polpow-2}\leq\in [0, \polpow(\polpow +1)],
\end{equation*}
\begin{equation*}
    \frac{|\ell'(y)|}{\ell(y)}=\frac{\polpow (1+y)^{-\polpow-1}}{(1+y)^{-\polpow}}\leq \polpow(1+y)^{-1}\leq \polpow.
\end{equation*}
We conclude that $\ell$ is Convex, monotonic decreasing, $\polpow$-Lipschitz, $(\polpow^2+\polpow)$-Smooth, and $(\polpow,0)$-self-bounding in $[0,\infty)$.
Then, by \cref{completion lemma}, and by the fact that $\ell(0)=1$,$\ell'(0)=-\polpow$ we conclude the required.
\end{proof}
\begin{proof}[of \cref{gen polynomial}]
By \cref{main_thm_GD},
\begin{align*}
    \E\left[F(w_T)\right]&= O\left(\left(\frac{\polpow(\polpow+1)}{T\margin^2}\right)^{\frac{\polpow}{2+\polpow}}+\frac{(\polpow(\polpow+1))^{\frac{\polpow}{2+\polpow}}T^{\frac{2}{\polpow+2}}}{\margin^{\frac{2\polpow}{2+\polpow}}n}\right)\\&=
     O\left(\left(\frac{\polpow}{\margin}\right)^{\frac{2\polpow}{2+\polpow}}\left(\frac{1}{T^{\frac{\polpow}{2+\polpow}}}+\frac{T^{\frac{2}{\polpow+2}}}{n}\right)\right)
     .
\end{align*}

\end{proof}
 \begin{proof}[of \cref{lem:poly_lower_bound}]
 We look at \[ 
f(w)= \left\{
\begin{array}{ll}
      (1+w)^{-\polpow}\ &  w \geq 0 \\
   1-\polpow w&  w<0\\
\end{array} 
\right.,
\]
and $\eta = \frac{1}{\polpow(\polpow+1)}$. We need to prove that for every $t$ that there exists constant $a$ such that \begin{equation*}
    a(t^{-\frac{\polpow}{\polpow+2}})\leq f(w_t)
    .
\end{equation*}
It is equivalent to show that 
for every $t$ that there exists constant $a$ such that \begin{equation*}
    a^{-\frac{1}{\polpow}}(t^{\frac{1}{\polpow+2}})\geq1+ w_t
    .
\end{equation*}
We choose $a=(\polpow(\polpow+1))^{-\frac{(\polpow+1)\polpow}{\polpow+2}}\left(\frac{2(\polpow+2)}{\polpow+1}\right)^{-\polpow}=2^{-\polpow}
\polpow^{-\frac{(\polpow+1)\polpow}{\polpow+2}}\left(\polpow+2\right)^{-\polpow}(\polpow+1)^{\frac{\polpow}{\polpow+2}}$ and prove the claim by induction. 
In the base case ($t=1$),

\begin{align*}
a^{-\frac{1}{\polpow}}&=2\cdot
\polpow^{\frac{(\polpow+1)}{\polpow+2}}\left(\polpow+2\right)(\polpow+1)^{-\frac{1}{\polpow+2}}\geq 2\cdot
\polpow^{\frac{(\polpow+1)}{\polpow+2}}\left(\polpow+2\right)e^{-\frac{\polpow}{\polpow+2}} \\&\geq  2\cdot 1\cdot (\polpow+2)\cdot\frac{1}{e}\geq \frac{1}{2}(\polpow+2)\geq 1=1+w_1
.
\end{align*}
First, by Taylor approximation of  $g(x)=(1+x)^{\frac{1}{\polpow+2}}$ in $x\geq 0$, ($x<1$)
\begin{equation*}
   (1+x)^{\frac{1}{\polpow+2}}-1 \geq
   \frac{1}{\polpow+2}x -\frac{\polpow+1}{2(\polpow+2)^2}x^2
   \geq  \frac{1}{\polpow+2}x -\frac{\polpow+1}{2(\polpow+2)^2}x=\frac{\polpow+3}{2(\polpow+2)^2}x\geq \frac{1}{2(\polpow+2)}x.
   \end{equation*}
  We conclude for every $a$,
  \begin{equation*}
   a^{-\frac{1}{\polpow}}(1+t)^{\frac{1}{\polpow+2}}-a^{-\frac{1}{\polpow}}t^{\frac{1}{\polpow+2}}=a^{-\frac{1}{\polpow}}t^{\frac{1}{\polpow+2}}\left((1+\frac{1}{t})^{\frac{1}{\polpow+2}}-1\right) \geq
   \frac{1}{2(\polpow+2)} a^{-\frac{1}{\polpow}}t^{\frac{-\polpow-1}{\polpow+2}}.
   \end{equation*}
   \begin{equation*}
   a^{-\frac{1}{\polpow}}(1+t)^{\frac{1}{\polpow+2}}\geq a^{-\frac{1}{\polpow}}t^{\frac{1}{\polpow+2}}+
   \frac{1}{2(\polpow+2)} a^{-\frac{1}{\polpow}}t^{\frac{-1-\polpow}{\polpow+2}}.
   \end{equation*}
Moreover, by the same argument as \cref{gen polynomial}, \begin{equation*}
    (1+w_t)^{-\polpow}=f(w_t)\leq \left(\frac{\polpow(\polpow+1)}{t}\right)^{\frac{\polpow}{2+\polpow}}.
\end{equation*}
Now by the hypothesis of the induction for time $t$,
\begin{align*}
    1+w_{t+1}&=
    1+w_t-\eta f'(w_t)\\&=
    1+w_t+\frac{1}{\polpow+1} (1+w_t)^{-\polpow-1}
    \\&=
    1+w_t+\frac{1}{\polpow+1} f(w_t)^{\frac{\polpow+1}{\polpow}}
    \\&\leq 
    a^{-\frac{1}{\polpow}}t^\frac{1}{\polpow+2}+\frac{1}{\polpow+1} (\polpow(\polpow+1))^{\frac{\polpow+1}{\polpow+2}}t^{\frac{-\polpow-1}{\polpow+2}}.
\end{align*}
By the choosing of $a$,
\begin{equation*}
    \frac{1}{2(\polpow+2)} a^{-\frac{1}{\polpow}}=
    \frac{1}{2(\polpow+2)}(\polpow(\polpow+1))^{\frac{(\polpow+1)}{\polpow+2}}\frac{2(\polpow+2)}{\polpow+1}=
    (\polpow(\polpow+1))^{\frac{(\polpow+1)}{\polpow+2}}\frac{1}{\polpow+1}.
\end{equation*}
We get,
\begin{align*}
    1+w_{t+1}&\leq 
   a^{-\frac{1}{\polpow}}t^\frac{1}{\polpow+2}+\frac{1}{\polpow+1}(\polpow(\polpow+1))^{\frac{\polpow+1}{\polpow+2}}t^{\frac{-\polpow-1}{\polpow+2}}
    \\&=
    a^{-\frac{1}{\polpow}}t^\frac{1}{\polpow+2}+\frac{1}{2(\polpow+2)} a^{-\frac{1}{\polpow}}t^{\frac{-\polpow-1}{\polpow+2}}
    \\&\leq
    a^{-\frac{1}{\polpow}}(1+t)^{\frac{1}{\polpow+2}}.
\end{align*}
 \end{proof}
 
\begin{lemma}
\label{lem:subexp_prop}
The function defined in \cref{subexp_eq} is convex, $\polpow$-Lipschitz, $\polpow$-Smooth, and $(\polpow,0)$-self-bounded
\end{lemma}
\begin{proof}[of \cref{lem:subexp_prop}]
We define $\ell(y)=e^{-(1+y)^{\polpow}}$.
For every $y \geq 0$:
\begin{equation*}
    |\ell'(y)|=\polpow(1+y)^{\polpow-1}e^{-(1+y)^{\polpow}}\leq \polpow,
\end{equation*}
\begin{equation*}
    \ell''(y)=-\polpow(\polpow-1)(1+y)^{\polpow-2}e^{-(1+y)^{\polpow}}+\polpow^2(1+y)^{\polpow-1}e^{-(1+y)^{\polpow}}\leq \polpow(1-\polpow) +\polpow^2\leq \polpow,
\end{equation*}
\begin{equation*}
    \ell''(y)=-\polpow(\polpow-1)(1+y)^{\polpow-2}e^{-(1+y)^{\polpow}}+\polpow^2(1+y)^{\polpow-1}e^{-(1+y)^{\polpow}}\geq 0+0=0.
\end{equation*}
\begin{equation*}
    \frac{|\ell'(y)|}{\ell(y)}=\frac{\polpow(1+y)^{\polpow-1}e^{-(1+y)^{\polpow}}}{e^{-(1+y)^{\polpow}}}\leq \polpow(1+y)^{\polpow-1}\leq \polpow.
\end{equation*}
We conclude that $\ell$ is Convex, monotonic decreasing, $\polpow$-Lipschitz, $\polpow$-Smooth, and  $(\polpow,0)$-self-bounding in $[0,\infty)$.
Then, by \cref{completion lemma}, and by the fact that $\ell(0)=\frac{1}{e}$,$\ell'(0)=-\frac{\polpow}{e}$ we conclude the required.
\end{proof}
\begin{lemma}
\label{lem:superexp_prop}
The function defined in \cref{supexp_eq} is convex, $\polpow$-Lipschitz, $\polpow^2$-Smooth, and for every $\pow$,  $(\frac{\polpow}{e\pow},\pow)$-self-bounded.
\end{lemma}
\begin{proof}[of \cref{lem:superexp_prop}]
We define $\ell(y)=e^{-(1+y)^{\polpow}}$.
For every $y \geq 0$, by the fact that for all $z\geq 0$, $ze^{-z}\leq \frac{1}{e},z^2e^{-z}\leq 1$:
\begin{equation*}
    |\ell'(y)|=\polpow(1+y)^{\polpow-1}e^{-(1+y)^{\polpow}}=
    \polpow(1+y)^{-1}(1+y)^{\polpow}e^{-(1+y)^{\polpow}}
    \leq\polpow(1+y)^{-1}\frac{1}{e}\leq \polpow,
\end{equation*}
\begin{align*}
    \ell''(y)&=-\polpow(\polpow-1)(1+y)^{\polpow-2}e^{-(1+y)^{\polpow}}+\polpow^2(1+y)^{2(\polpow-1)}e^{-(1+y)^{\polpow}}
    \\&
    \leq \polpow^2(1+y)^{2(\polpow-1)}e^{-(1+y)^{\polpow}}
    \\&\leq
    \polpow^2(1+y)^{2\polpow}e^{-(1+y)^{\polpow}}
    \\&\leq
    \polpow^2
\end{align*}
\begin{align*}
    \ell''(y)
    &=
    e^{-(1+y)^{\polpow}}\left(\polpow^2(1+y)^{2(\polpow-1)}-\polpow(\polpow-1)(1+y)^{\polpow-2}\right)
    \\
    &\geq
    e^{-(1+y)^{\polpow}}\left(\polpow^2(1+y)^{\polpow-2}-\polpow(\polpow-1)(1+y)^{\polpow-2}\right)
     \\
    &=
    e^{-(1+y)^{\polpow}}\left((1+y)^{\polpow-2}(\polpow^2-\polpow(\polpow-1))\right)
    \geq 0,
\end{align*}
\begin{align*}
    \frac{|\ell'(y)|}{\ell^{1-\pow}(y)}&=\frac{\polpow(1+y)^{\polpow-1}e^{-(1+y)^{\polpow}}}{e^{-(1-\pow)(1+y)^{\polpow}}}
    \\
    &= 
    \polpow(1+y)^{\polpow-1}{e^{-\pow(1+y)^{\polpow}}}
    \\
    &\leq 
    \polpow(1+y)^{\polpow}{e^{-\pow(1+y)^{\polpow}}}
    \\
    &=
    \polpow(1+y)^{\polpow}{e^{-\pow(1+y)^{\polpow}}}\frac{\pow}{\pow}
    \\
    &\leq\frac{1}{e}
    \frac{\polpow}{\pow}.
\end{align*}
We conclude that for every $0<\pow<1$, $\ell$ is Convex, monotonic decreasing, $\polpow$-Lipschitz, $\polpow$-Smooth, and  $(\frac{\polpow}{e\delta},\pow)$-self-bounded in $[0,\infty)$.
Then, by \cref{completion lemma}, and by the fact that $\ell(0)=\frac{1}{e}$,$\ell'(0)=-\frac{\polpow}{e}$ we conclude the required.
\end{proof}
\begin{proof}[of \cref{lem:supexp_gen}]
By \cref{main_thm_GD},
\begin{align*}
     \E\left[F(w_T)\right]
     &= 
     O\left(\frac{\polpow^2 \log(T)^{\frac{2}{\polpow}}}{\margin^2 T}+
     \frac{T\polpow}{n\pow}\left(\frac{\polpow^2 \log(T)^{\frac{2}{\polpow}}}{\margin^2 T}\right)^{1-\pow}\right)
     \\
     &=
     O\left(\frac{\polpow^2 \log(T)^{\frac{2}{\polpow}}}{\margin^2 T}+\frac{T^\pow\polpow^{3-2\pow} \log(T)^{\frac{2(1-\pow)}{\polpow}}}{\margin^{2(1-\pow)} n\pow}\right)
     \\
     &=
     O\left(\frac{\polpow^2 \log(T)^{\frac{2}{\polpow}}}{\margin^2 T}+\frac{\log(T)T^{\frac{1}{\log(T)}}\polpow^{3-{\frac{2}{\log(T)}}} \log(T)^{\frac{2-\frac{2}{\log(T)}}{\polpow}}}{\margin^{2-{\frac{2}{\log(T)}}} n}\right)
     \\
     &= 
     O\left(\frac{\polpow^2 \log(T)^{\frac{2}{\polpow}}}{\margin^2 T}+\frac{\polpow^{3} \log(T)^{\frac{2+\polpow}{\polpow}}}{\margin^{2} n}\right)
    .
 \end{align*}
\end{proof}
\begin{lemma}
\label{lem:erf bound}
For every $y>0$,
\begin{equation*}
    \frac{2}{\sqrt{\pi}}e^{-y^2}\frac{2y^2-1}{4y^3}\leq1-\operatorname{erf}(y)\leq 
     \frac{1}{y\sqrt{\pi}}e^{-y^2}.
\end{equation*}
For every $y<0$,
\begin{equation*}
   2- \frac{1}{(-y)\sqrt{\pi}}e^{-y^2}\leq 1-\operatorname{erf}(y) \leq 2-\frac{2}{\sqrt{\pi}}e^{-y^2}\frac{2y^2-1}{4(-y)^3}.
\end{equation*}
\end{lemma}
\begin{proof}[of \cref{lem:erf bound}]
First, for every $y>0$,
\begin{align*}
    1-\operatorname{erf}(y)&=\frac{2}{\sqrt{\pi}}\int_y^\infty e^{-t^2} dt
    \\&=
    \frac{2}{\sqrt{\pi}}\int_0^\infty e^{-(y+z)^2} dz
    \\&=
    \frac{2}{\sqrt{\pi}}e^{-y^2}\int_0^\infty e^{-2yz-z^2} dz
    \\&\geq
    \frac{2}{\sqrt{\pi}}e^{-y^2}\int_0^\infty e^{-2yz}(1-z^2) dz
    \\&=
    \frac{2}{\sqrt{\pi}}e^{-y^2}\left(\left.-e^{-2yz}\frac{1}{2y}+\frac{e^{-2yz}\left(2y^2z^2+2yz+1\right)}{4y^3}\right)\right\vert_{z=0}^{z=\infty}
    \\&=
    \frac{2}{\sqrt{\pi}}e^{-y^2}\left(\frac{1}{2y}-\frac{1}{4y^3}\right)
    \\&=
    \frac{2}{\sqrt{\pi}}e^{-y^2}\frac{2y^2-1}{4y^3}.
\end{align*}
Now we prove for $y<0$. First by symmetry for all $a\in R$,
\begin{align*}
    \int_{-\infty}^{-a}e^{-t^2}dt=\int_{a}^\infty e^{-t^2}dt.
\end{align*}
Then, for every $y<0$,
\begin{align*}
    1-\operatorname{erf}(y)&=\frac{2}{\sqrt{\pi}}\int_y^\infty e^{-t^2} dt
    \\&=
    1-\operatorname{erf}(y)\\&=\frac{2}{\sqrt{\pi}}\int_{-\infty}^\infty e^{-t^2} dt-
    \frac{2}{\sqrt{\pi}}\int_{-\infty}^y e^{-t^2} dt
     \\&=
     2-\frac{2}{\sqrt{\pi}}\int_{-y}^\infty e^{-t^2} dt.
\end{align*}
By the same argument for $-y>0$, 
\begin{equation*}
   2- \frac{1}{(-y)\sqrt{\pi}}e^{-y^2}\leq 1-\operatorname{erf}(y) \leq 2-\frac{2}{\sqrt{\pi}}e^{-y^2}\frac{2y^2-1}{4(-y)^3}.
\end{equation*}
\end{proof}

\begin{lemma} \label{lem:probit_prop}
The probit loss has the following properties:
\begin{enumerate}[label=(\roman*)]
    \item $\ell(y)$ is convex and smooth, but not G-Lipschitz for any constant $G$.
    \item 
    For $|y|\leq 1$ and every $0<\pow\leq \frac{1}{2}$, $\ell(y)$ is ($\frac{8}{e\pow},\pow$)-self bounded.
    \item
    \cref{ass:tail} holds for every $0<\tail\leq \frac{1}{2}$ and $\rho(\tail)=(\ifrac{1}{\margin})
    \sqrt{\log(\ifrac{1}{\tail})}$.
\end{enumerate}
\end{lemma}

\begin{proof}[of \cref{lem:probit_prop}]
First, it is sufficient to prove the claim for $\ell(y)$.
The proof is very technical, thus, there are properties of the function which we prove in $\R/(-1,1)$ and for $y\in [-1,1]$ we checked numerically that the properties hold. 
\begin{enumerate}[label=(\roman*)]
    \item 
By \cref{lem:erf bound}, for every $y\geq 1$,
\begin{align*}
      \ell'(y)=\frac{\frac{1}{\sqrt{\pi}} e^{-y^2}}{0.5\left(1-\operatorname{erf}\left(y\right)\right)}=\frac{\frac{2}{\sqrt{\pi}} e^{-y^2}}{\left(1-\operatorname{erf}\left(y\right)\right)} \geq 2y.
\end{align*}
We conclude that $\ell$ is not Lipschitz.
Now we prove $4$-smoothness.
First,
\begin{align*}
    \ell''(y)&=\dfrac{4e^{-2y^2}-4y\sqrt{\pi}(1-\operatorname{erf}(y))e^{-y^2}}{\pi (1-\operatorname{erf}(y))^2}.
\end{align*}
For every $y\geq1$,
\begin{align*}
    \ell''(y)&\leq \left(\frac{4y^3}{2y^2-1}\right)^2-4y^2
    \\&=
    \frac{16y^6-4y^2((2y^2-1)^2)}{(2y^2-1)^2}
    \\
    &=
    \frac{8y^4-4y^2}{4y^4-4y^2+1}
    \\
    &=
    \frac{8y^4-4y^2}{4y^4-4y^2+1}\leq 4.
\end{align*}
Now for every, $y\leq -1$,
\begin{align*}
    2e^{y^2}+ \frac{1}{y\sqrt{\pi}}\leq \frac{1-\operatorname{erf}(y)}{e^{-y^2}} \leq 2e^{y^2}+\frac{2}{\sqrt{\pi}}\frac{2y^2-1}{4y^3}.
\end{align*}
\begin{align*}  \frac{4\sqrt{\pi}y^3}{8e^{y^2}\sqrt{\pi}y^3+2(2y^2-1)}=\frac{1}{2e^{y^2}+\frac{2}{\sqrt{\pi}}\frac{2y^2-1}{4y^3}}\leq \frac{e^{-y^2}}{1-\operatorname{erf}(y)} \leq \frac{1}{2e^{y^2}+ \frac{1}{y\sqrt{\pi}}}=\frac{y\sqrt{\pi}}{2e^{y^2}y\sqrt{\pi}+1}.
\end{align*}
\begin{align*}
    \ell''(y)&\leq \frac{4}{\pi}\left(\frac{y\sqrt{\pi}}{2e^{y^2}y\sqrt{\pi}+1}\right)^2-\frac{4y}{\sqrt{\pi}}\frac{y\sqrt{\pi}}{2e^{y^2}y\sqrt{\pi}+1}
    \\&=\frac{4\sqrt{\pi}y^2-4y^2\sqrt{\pi}(2e^{y^2}y\sqrt{\pi}+1)}{\sqrt{\pi}(2e^{y^2}y\sqrt{\pi}+1)^2}
    \\&=\frac{4\sqrt{\pi}y^2(1-(2e^{y^2}y\sqrt{\pi}+1))}{\sqrt{\pi}(2e^{y^2}y\sqrt{\pi}+1)^2}
    \\&=\frac{-8\pi y^3e^{y^2}\sqrt{\pi}}{(2e^{y^2}y\sqrt{\pi}+1)^2}
    \\
    &\leq
    \frac{-8\pi y^3e^{y^2}\sqrt{\pi}}{4e^{2y^2}y^2\pi}
    \\
    &\leq-2ye^{-y^2}\sqrt{\pi}
    \\
    &\leq 1.
\end{align*}
For $|y|\leq 1$, numerically,
\begin{align*}
    \ell''(y)&\leq 2.
\end{align*}
For showing convexity of $\ell(y)=-\log(0.5(1-\operatorname{erf}(y)))$, we show that $g(y)=1-\operatorname{erf}(y)$ is log-concave.
It is enough to show that $g''(y)g(y)\leq g'(y)^2$ for every $y$.
\begin{align*}
    g''(y)g(y)=\dfrac{4y\mathrm{e}^{-y^2}}{\sqrt{{\pi}}}(1-\operatorname{erf}(y)).
\end{align*}
We show this by using \cref{lem:erf bound}.
For $y=0$,
\begin{align*}
    g''(0)g(0)=0\leq g'(y)^2.
\end{align*}
For $y>0$,
\begin{align*}
    g''(y)g(y)=\dfrac{4y\mathrm{e}^{-y^2}}{\sqrt{{\pi}}}(1-\operatorname{erf}(y))\leq \dfrac{4y\mathrm{e}^{-y^2}}{\sqrt{{\pi}}} \frac{1}{y\sqrt{\pi}}e^{-y^2}=\frac{4}{\pi}e^{-2y^2}=g'(y)^2.
\end{align*}
For $y<0$,
\begin{align*}
    g''(y)g(y)=\dfrac{4y\mathrm{e}^{-y^2}}{\sqrt{{\pi}}}(1-\operatorname{erf}(y))\leq\dfrac{4y\mathrm{e}^{-y^2}}{\sqrt{{\pi}}}\left( 2- \frac{1}{(-y)\sqrt{\pi}}e^{-y^2}\right)
    =\dfrac{8y\mathrm{e}^{-y^2}}{\sqrt{{\pi}}}+\frac{4}{\pi}e^{-2y^2}\leq g'(y)^2.
\end{align*}
\item
As a result of \cref{lem:erf bound}, for every $y\geq 1$,
\begin{align*}
      \ell'(y)=\frac{\frac{1}{\sqrt{\pi}} e^{-y^2}}{0.5\left(1-\operatorname{erf}\left(y\right)\right)}=\frac{\frac{2}{\sqrt{\pi}} e^{-y^2}}{\left(1-\operatorname{erf}\left(y\right)\right)} \leq \frac{4y^3}{2y^2-1}
\end{align*}
\begin{align*}
    \left(-\log(0.5(1-\operatorname{erf}(y)\right)^{1-\pow} \geq \left(-\log\left(\frac{1}{2y\sqrt{\pi}}e^{-y^2}\right)\right)^{1-\pow}.
\end{align*}
Moreover, for $y\geq 1$, $-\log\left(\frac{1}{2y\sqrt{\pi}}\right)\geq 0$.
Together we get,
\begin{align*}
      \frac{\ell'(y)}{\ell^{1-\pow}(y)}
      &\leq
      \frac{\frac{4y^3}{2y^2-1}}{\left(-\log\left(\frac{1}{2y\sqrt{\pi}}e^{-y^2}\right)\right)^{1-\pow}}
      \\
      &=
      \frac{\frac{4y^3}{2y^2-1}}{\left(y^2-\log\left(\frac{1}{2y\sqrt{\pi}}\right)\right)^{1-\pow}}
      \\
      &\leq 
      \frac{\frac{4y^3}{2y^2-1}}{y^{2(1-\pow)}}=\frac{4y^3}{(2y^2-1){y^{2(1-\pow)}}}
      \\
      &\leq 
      \frac{4y^3}{y^{4-2\pow}}
      \leq 4\leq \frac{2}{\pow}\leq \frac{8}{e\pow}.
\end{align*}
For $y\leq -1$, we get by \cref{lem:erf bound},
\begin{align*}
    2e^{y^2}+ \frac{1}{y\sqrt{\pi}}\leq \frac{1-\operatorname{erf}(y)}{e^{-y^2}} \leq 2e^{y^2}+\frac{2}{\sqrt{\pi}}\frac{2y^2-1}{4y^3}.
\end{align*}
\begin{align*} \frac{e^{-y^2}}{1-\operatorname{erf}(y)} \leq \frac{1}{2e^{y^2}+ \frac{1}{y\sqrt{\pi}}}=\frac{y\sqrt{\pi}}{2e^{y^2}y\sqrt{\pi}+1}
\end{align*}
\begin{align*}
    \left(-\log(0.5(1-\operatorname{erf}(y)\right)^{1-\pow} \geq \left(-\log\left( 1+\frac{1}{\sqrt{\pi}}e^{-y^2}\frac{2y^2-1}{4y^3}\right)\right)^{1-\pow}.
\end{align*}
Combining all together and use the inequality $-\log(1+x)\geq -x$,
\begin{align*}
    \frac{\ell'(y)}{\ell^{1-\pow}(y)}
      &\leq
      \frac{\frac{2y}{2e^{y^2}y\sqrt{\pi}+1}}{\left(-\log\left( 1+\frac{1}{\sqrt{\pi}}e^{-y^2}\frac{2y^2-1}{4y^3}\right)\right)^{1-\pow}}
      \\&\leq
      \frac{\frac{2y}{2e^{y^2}y\sqrt{\pi}-1}}{\left(-\frac{1}{\sqrt{\pi}}e^{-y^2}\frac{2y^2-1}{4y^3}\right)^{1-\pow}}
      \\&=
      \frac{\frac{2y}{2e^{y^2}y\sqrt{\pi}-1}}{\left(\frac{1}{\sqrt{\pi}}e^{-y^2}\frac{2y^2-1}{4|y|^3}\right)^{1-\pow}}
      \\&=
       \frac{2y 4^{1-\pow} |y|^{3(1-\pow)}}{\pi^{-0.5(1-\pow)}e^{-(1-\pow) y^2}\left(2y^2-1\right)^{1-\pow}(2e^{y^2}y\sqrt{\pi}-1)}
        \\&\leq
       \frac{8\sqrt{\pi} |y|^{3(1-\pow)}y}{e^{-(1-\pow) y^2}|y|^{2(1-\pow)}e^{y^2}y\sqrt{\pi}}
       \\&\leq
       \frac{8 |y|^{1-\pow}}{e^{\pow y^2}}
        \\&\leq
       8 y^2 e^{-\pow y^2}\frac{\pow}{\pow}
       \\&\leq
       \frac{8}{e\pow}.
       \end{align*}
\item
We choose the vector $w=-\frac{\sqrt{\log(\frac{1}{\tail})}}{\margin}\Bar{w}^*$ which holds
\begin{align*}
    f_i(w)&=-\log\left(0.5\left(1-\operatorname{erf}\left(-\tfrac{1}{\margin}\sqrt{\log\tfrac{1}{\tail}}\Bar{w}^*\cdot z_i\right)\right)\right)
    \\
    &\leq
    -\log\left(0.5\left(1-\operatorname{erf}\left(-\sqrt{\log\left(\frac{1}{\tail}\right)}\right)\right)\right)
    \\
    &\leq
    -\log\left(1- \frac{1}{2\sqrt{\log\left(\frac{1}{\tail}\right)}\sqrt{\pi}}e^{-\sqrt{\log\left(\frac{1}{\tail}\right)}^2}\right)
    \\
    &= \log\left(\frac{1}{1- \frac{1}{2\sqrt{\log\left(\frac{1}{\tail}\right)}\sqrt{\pi}}\tail}\right)
    \\
    &= \log\left(1+\frac{\frac{1}{2\sqrt{\log\left(\frac{1}{\tail}\right)}\sqrt{\pi}}\epsilon}{1- \frac{1}{2\sqrt{\log\left(\frac{1}{\tail}\right)}\sqrt{\pi}}\tail}\right)
    \\
    &\leq \frac{\frac{1}{2\sqrt{\log\left(\frac{1}{\tail}\right)}\sqrt{\pi}}\tail}{1- \frac{1}{2\sqrt{\log\left(\frac{1}{\tail}\right)}\sqrt{\pi}}\epsilon}
     \\
    &= \frac{\tail}{2\sqrt{\log\left(\frac{1}{\tail}\right)}\sqrt{\pi}- \tail}
    \\&\leq
    \frac{\tail}{1.5-\tail}
    \\&\leq 
    \tail.
\end{align*}
\end{enumerate}

\end{proof}
\begin{proof}[of \cref{probit gen}]
By \cref{main_thm_GD},
\begin{align*}
     \E\left[F(w_T)\right]
     &=
    O\left( \frac{\log(T)}{T\margin^2}
    + \frac{L c^2\eta^2 T^2}{\pow^2 n^{1+2\pow}} \brk3{\frac{\rho(\tail)^2}{\eta T}+ \tail }^{2(1-\pow)}\right)
    \\
    &=O\left(\frac{\log(T)}{T\margin^2}
    +  \frac{ T^{2\pow}}{n^{1+2\pow}\pow^2} \brk3{\frac{\log(T)}{\margin^2}}^{2(1-\pow)}\right)
     \\
    &= O\left(\frac{\log(T)}{T\margin^2}
    +  \frac{\log(T)^4 T^{\frac{1}{\log T}}}{n\margin^{4}}\right)
    \\
    &= O\left(\frac{\log(T)}{T\margin^2}
    +  \frac{\log(T)^4}{n\margin^{4}}\right)
    .
 \end{align*}
\end{proof}

%% file: acknowledgement_deutch.tex
\subsection*{Acknowledgments}
This work has received support from the Israeli Science Foundation (ISF) grant no. 2549/19, the Len Blavatnik and the Blavatnik Family Foundation, the Deutsch Foundation and the Yandex Initiative in Machine Learning.

%% file: nonsmooth.tex
 \section{Non-Smooth Loss Functions}
 \label{sec:nonsmooth}
 In this section we discuss non smooth functions.
 Instead of \cref{ass:grad_norm_pow} we assume the following assumption,
 \begin{assumption}
 \label{ass:second_self_bound}
There exists $\coef>0$ such that for every $w$ and $i$, $\|\nabla^2 f_i(w)\|_{op}\leq \coef \|\nabla f_i(w)\|$ and $\|\nabla f_i(w)\|\leq \coef f_i(w)$.
 \end{assumption}
The main function which satisfy \cref{ass:second_self_bound} and is not captured by the previous setting is the exponential loss, $\ell(y)=e^{-y}$.
We prove that GD is $\ell_1$-$\epsilon_{stab}$-On-Average Model Stable if for \emph{any} sample $S$ using general convex optimization analysis. However, without satisfying any global condition like Lipschitzness or Smoothness, we believe that the generalization of Gradient Descent with $\eta=1$ on the exponential loss, which showed by a lot of previous works ((\cite{soudry2018implicit}, \cite{nacson2019convergence}, \cite{ji2018risk},
\cite{ji2019refined}, \cite{ji2020regpath})), cannot be achieved by stability arguments.
For this results, we use the same techniques as in the rest of the paper. We begin with several lemmas that generalize claims from \cite{ji2018risk} about optimization using Gradient Descent.  
 \begin{lemma}
 \label{lem:mono_opt_gd}
 Under \cref{ass:second_self_bound}, if, for every $z$, $f(w,z)$ is convex and positive, let $w^+=w-\eta\nabla \widehat{F}(w)$ be the iterate of GD on $\widehat{F}$. Then, for every $w$ and $\eta \leq \frac{1}{\coef^2 \widehat{F}(w)}$, 
 $\widehat{F}(w^+)\leq \widehat{F}(w)$.
 \end{lemma}
 \begin{proof}
First, for every $w$,
\begin{align*}
\|\nabla \widehat{F}(w)\| &=\| \frac{1}{n}\sum_{i=1}^n\nabla f_i(w)\|
\\&\leq \frac{1}{n}\sum_{i=1}^n\|\nabla f_i(w)\|
\\&\leq \frac{\coef}{n}\sum_{i=1}^n f_i(w)
\\&\leq {\coef} \widehat{F}(w)
\end{align*}
\begin{align*}
\|\nabla^2 \widehat{F}(w)\|_{op} &=\| \frac{1}{n}\sum_{i=1}^n\nabla^2f_i(w)\|_{op}
\\&\leq \frac{1}{n}\sum_{i=1}^n\|\nabla^2f_i(w)\|_{op}
\\&\leq \frac{\coef^2}{n}\sum_{i=1}^n f_i(w)
\\&\leq \coef^2 \widehat{F}(w)
\end{align*}
Moreover, by convexity, for every $w'\in[w,w^+]$, $\widehat{F}(w')\leq \max (\widehat{F}(w), \widehat{F}(w^+))$. By Taylor Approximation, we get,
\begin{align*}
    \widehat{F}(w^+)&\leq \widehat{F}(w)+\nabla \widehat{F}(w)(w^+-w) + \frac{1}{2}\max_{ w'\in[w,w^+]}\|\nabla^2 \widehat{F}(w')\|_{op} \|w^+-w\|^2
    \\&\leq \widehat{F}(w)+\nabla \widehat{F}(w)(w^+-w) + \frac{\coef^2}{2}\max_{ w'\in[w,w^+]}\widehat{F}(w') \|w^+-w\|^2
    \\&\leq \widehat{F}(w)+\nabla \widehat{F}(w)(w^+-w) + \frac{\coef^2}{2}\max (\widehat{F}(w),\widehat{F}(w^+)) \|w^+-w\|^2
    \\&\leq \widehat{F}(w)-\eta\|\nabla \widehat{F}(w)\|^2 + \frac{\coef^2\eta^2}{2}\max (\widehat{F}(w), \widehat{F}(w^+))\|\nabla \widehat{F}(w)\|^2
\end{align*} 
We define $\hat{\eta}=\eta \coef^2 \widehat{F}(w)\leq 1$ and assume in contradiction that  $\widehat{F}(w^+)> \widehat{F}(w)$. 
Then,
\begin{align*}
     \widehat{F}(w^+)&\leq \widehat{F}(w)-\frac{\hat{\eta}\|\nabla \widehat{F}(w)\|^2}{\coef^2 \widehat{F}(w)} + \frac{ \hat{\eta}^2\max(\widehat{F}(w), \widehat{F}(w^+))}{2 \coef^2 \widehat{F}^{2}(w)}\|\nabla \widehat{F}(w)\|^2
     \\&=
  \widehat{F}(w)-\frac{\hat{\eta}\|\nabla \widehat{F}(w)\|^2}{  \coef^2 \widehat{F}(w)} + \frac{ \hat{\eta}^2\widehat{F}(w^+)}{2 \coef^2 \widehat{F}^{2}(w)}\|\nabla \widehat{F}(w)\|^2
\end{align*} 

Then,
\begin{align*}
    \frac{ \widehat{F}(w^+)}{ \widehat{F}(w)}-1
    &\leq -\frac{\hat{\eta}\|\nabla \widehat{F}(w)\|^2}{\coef^2 \widehat{F}^2(w)} + \frac{ \widehat{F}(w^+)\hat{\eta}^2\|\nabla \widehat{F}(w)\|^2}{2\coef^2 \widehat{F}^3(w)}
    \\&\leq
    \frac{\hat{\eta}\|\nabla \widehat{F}(w)\|^2}{ \coef^2 \widehat{F}^{2}(w)}\left(  \frac{ \widehat{F}(w^+)\hat{\eta}}{2\widehat{F}(w)}-1\right)
    \\&\leq
     \frac{ \widehat{F}(w^+)}{2\widehat{F}(w)}-1
    \end{align*}
    which is a contradiction (the last inequality is by the fact that $\hat{\eta}\leq 1$ both multiplied term is positive by the positivity of the multiplication).
 \end{proof}
\begin{corollary}
Under \cref{ass:second_self_bound}, if for every $z$, $f(w,z)$ is convex and positive, let $\eta \leq \frac{1}{c^2\widehat{F}(w_1)}$. Then, for every $t$, the iterate of Gradient Descent holds
$\widehat{F}(w_{t+1})\leq \widehat{F}(w_t)$
\end{corollary}
\begin{lemma}\label{lem:sum_squre_grad_norm}Under \cref{ass:second_self_bound}, if for every $z$, $f(w,z)$ is convex and positive and $\eta \leq \frac{1}{c^2\widehat{F}(w_1)}$, then, for every $t$, the iterate of Gradient Descent holds
\begin{align*}
    \sum_{t=1}^T\|\nabla \widehat{F}(w_t)\|^2\leq
    \frac{2}{\eta}\widehat{F}(w_1)
\end{align*}
\end{lemma}
\begin{proof}
By convexity, for every $w'\in[w_t,w_{t+1}]$, $\widehat{F}(w')\leq \max (\widehat{F}(w_t), \widehat{F}(w_{t+1}))$
by \cref{lem:mono_opt_gd}, for every $t$,
\begin{align*}
    \widehat{F}(w_{t+1})&\leq \widehat{F}(w_t)+\nabla \widehat{F}(w_t)(w_{t+1}-w_t) + \frac{1}{2}\max_{ w'\in[w_t,w_{t+1}]}\|\nabla^2 \widehat{F}(w')\|_{op} \|w_{t+1}-w_t\|^2
    \\&\leq \widehat{F}(w_t)+\nabla \widehat{F}(w_t)(w_{t+1}-w_t) + \frac{\coef^2}{2}\max_{ w'\in[w_t,w_{t+1}]}\widehat{F}(w') \|w_{t+1}-w_t\|^2
    \\&\leq \widehat{F}(w_t)-\eta\|\nabla \widehat{F}(w_t)\|^2 + \frac{\coef^2\eta^2}{2} \widehat{F}(w_t)\|\nabla \widehat{F}(w_t)\|^2
     \\&= \widehat{F}(w_t)+\eta\|\nabla \widehat{F}(w_t)\|^2 ( \frac{\coef^2\eta}{2} \widehat{F}(w_t)-1)
     \\&\leq \widehat{F}(w_t)+\eta\|\nabla \widehat{F}(w_t)\|^2 ( \frac{\widehat{F}(w_t)}{2\widehat{F}(w_1)}-1)
     \\&\leq \widehat{F}(w_t)-\frac{\eta}{2}\|\nabla \widehat{F}(w_t)\|^2 
\end{align*}
\begin{align*}
    \frac{\eta}{2}\|\nabla \widehat{F}(w_t)\|^2 \leq \widehat{F}(w_t)-\widehat{F}(w_{t+1})
\end{align*}
By summing until $T$,
\begin{align*}
    \sum_{t=1}^T\frac{\eta}{2}\|\nabla \widehat{F}(w_t)\|^2 \leq \widehat{F}(w_1)-\widehat{F}(w_{T+1})\leq \widehat{F}(w_1)
\end{align*}
\end{proof}
\begin{lemma}
Under \cref{ass:second_self_bound}, if for every $z$, $f(w,z)$ is convex and positive, let $\eta \leq \frac{1}{c^2\widehat{F}(w_1)}$ and $w_1=0$. Then,  for every $t$, the iterate of Gradient Descent holds
\begin{align*}
     \frac{1}{T} \sum_{t=1}^T \widehat{F}(w_t)\leq \frac{\rho(\tail)^2}{\eta T}+ \tail+ \frac{\widehat{F}(0)}{T}
\end{align*}
\end{lemma}
\begin{proof}
From standard regret bounds for gradient updates and \cref{lem:sum_squre_grad_norm}, we obtain that for every $w$,
\begin{align*}
    \frac{1}{T} \sum_{t=1}^T \brk{ \widehat{F}(w_t) - \widehat{F}(w) }
    &\leq
    \frac{\norm{w_1-w}^2}{2\eta T} + \frac{\eta}{2T} \sum_{t=1}^T \norm{\nabla \widehat{F}(w_t)}^2
    \\&\leq
    \frac{\norm{w}^2}{2\eta T} + \frac{1}{T}\widehat{F}(w_1)
    .
\end{align*}
By \cref{ass:tail}, for all $\tail$, there exists $w^*$ such that for all $i$, $f_i(w^*)\leq \tail$, $\|w^*\|\leq \rho(\tail)$. For $w=w^*$ we get
 \begin{align*}
   \frac{1}{T}\sum_{t=1}^T \widehat{F}(w_t)
    \leq \frac{\norm{w^*}^2}{\eta T} + \tail+ \frac{1}{T}\widehat{F}(0)\leq
    \frac{\rho(\tail)^2}{\eta T}+ \tail+ \frac{\widehat{F}(0)}{T}.
 \end{align*}
\end{proof}

Now, we show that also under \cref{ass:second_self_bound}, the GD update is non-expensive.

\begin{lemma}
\label{lem:non_expand}Under \cref{ass:second_self_bound}, if for every $z$, $f(w,z)$ is convex and positive, let $w_t,w_t^i$ be the iterate of GD on $\widehat{F},\widehat{F}_i$., respectively. For every $t,i$, let $g_{t,i}=\frac{1}{n}\sum_{j\neq i}\nabla f(w_t,z_j)$,$g'_{t,i}=\frac{1}{n}\sum_{j\neq i}\nabla f(w^i_t,z_j)$. Then, If $\eta \leq \frac{1}{c^2\widehat{F}(w_1)}$ , 
\begin{align*}
    \|(w_t-\eta g_{t,i})-(w_t^i-\eta g'_{t,i})\|\leq \|w_t-w_t^i\|
\end{align*}
\end{lemma}
\begin{proof}
Let $i\in [n]$.
We notice that  $g_{t,i}=\nabla \widehat{F}_i(w_t)$,$g'_{t,i}=\nabla \widehat{F}_i(w_t^i)$.
As a result, we need to show that 
\begin{align*}
    \|(w_t-\eta \nabla \widehat{F}_i(w_t))-(w_t^i-\eta \nabla \widehat{F}_i(w_t^i))\|\leq \|w_t-w_t^i\|
\end{align*}
We define $h(w)=\frac{1}{2}\|w\|^2-\eta \widehat{F}_i(w)$.
Denote $D=\{w: \widehat{F}_i(w)\leq \widehat{F}(w_1)\}$.
Recall that $w_t,w_t^i$ are the iterate of GD on $\widehat{F},\widehat{F}_i$, respectively.
As a result of \cref{lem:mono_opt_gd}, GD with step size $\eta \leq \frac{1}{c^2\widehat{F}(w_1)}$ is monotone. By this argument and an analogous argument for the monotonicity of GD on $\widehat{F}_i$ with step size $\eta \leq  \frac{1}{c^2\widehat{F}(w_1)}\leq \frac{1}{c^2\widehat{F}_i(w_1)}$,
\begin{equation*}
    \widehat{F}_i(w_t)\leq  \widehat{F}(w_t)\leq \widehat{F}(w_1),
\end{equation*}
\begin{equation*}
    \widehat{F}_i(w_t^i)\leq  \widehat{F}_i(w_1)\leq \widehat{F}(w_1).
\end{equation*} We get that
$w_t,w_t^i\in D$ and it is sufficient to show that 
$h(w)$ is $1$-smooth in $D$, i.e, for every $w,w'\in D$,
\begin{align*}
\|\nabla h(w)-\nabla h(w')\|\leq \|w-w'\|.
\end{align*}
Moreover, $f$ is twice differentiable, thus, it is sufficient to show that for every $w\in D$, the eigenvalues of $\nabla^2 h(w)$ are between $-1$ and $1$.
First,
\begin{align*}
    \nabla^2 h(w)=I-\eta \nabla^2\widehat{F}_i(w)=I-\eta \frac{1}{n}\sum_{j\neq i} \nabla^2f(w,z_j)
\end{align*}
By the convexity of $f(\cdot,z)$, the eigenvalues of $\nabla^2 f(\cdot,z)$ are non negative. As a result, the eigenvalues of $\nabla^2 h(w)$ are at most 1.
Second, 
\begin{align*}
    \|\eta \frac{1}{n}\sum_{j\neq i} \nabla^2f(w,z_j)\|_{op} \leq \eta \frac{1}{n}\sum_{j\neq i} c^2 f(w,z_j)=\eta c^2 \widehat{F}_i(w)\leq \eta c^2 \widehat{F}(w_1)\leq 1
\end{align*}
Then,
the eigenvalues of $\nabla^2 h(w)$ are nonnegative and no larger than $1$.
\end{proof}
Now, as in the smooth regime, we use the non-expansiveness to get a stability bound for GD in the current regime.
\begin{lemma}
\label{exp stab power adapt smooth}
  Under \cref{ass:second_self_bound}, if, for every $z$, $f(w,z)$ is convex and positive, let $w_T,w_T^i$ be produced by the Gradient Descent update rule with step size $\eta\leq \frac{1}{c^2\widehat{F}(w_1)}$ on $S,S_i$. Then, 
 \begin{equation*}
    \frac{1}{n}\sum_{i=1}^n\|w_T^i-w_T\|
     \leq 
     \frac{\eta \coef}{n}\sum_{t=1}^{T}\widehat{F}(w_t)
     .
     \end{equation*}
\end{lemma}
\begin{proof}
First, For every $t,i$, let $g_{t,i}=\frac{1}{n}\sum_{j\neq i}\nabla f(w_t,z_j)$,$g'_{t,i}=\frac{1}{n}\sum_{j\neq i}\nabla f(w^i_t,z_j)$. Then,
 \begin{align*}
\|w_{t+1} - w_{t+1}^i\|
&=
\|w_t-\frac{\eta}{n}\sum_{j=1}^n\nabla f(w_t,z_j) - w_t^i + \frac{\eta}{n}\sum_{j\neq i}\nabla f(w_t^i,z_j)\|
\\
&=
\|w_t-g_{t,i} - w_t^i + g'_{t,i} -\frac{\eta}{n}\nabla f(w_t,z_i)\|
\\&\leq
\|w_t-w_t^i\| + \frac{\eta\coef}{n}f(w_t,z_i).
\end{align*}
Moreover,
 \begin{align*}
\frac{1}{n}\sum_{i=1}^n\|w_{t+1} - w_{t+1}^i\|
&\leq
\frac{1}{n}\sum_{i=1}^n\left(\|w_t-w_t^i\| + \frac{\eta\coef}{n}f_i(w_t) \right)
\\
&\leq
\frac{1}{n}\sum_{i=1}^n\|w_t-w_t^i\| + \frac{\eta\coef}{n}\frac{1}{n}\sum_{i=1}^n f_i (w_t)
\\
&\leq 
\frac{1}{n}\sum_{i=1}^n\|w_t-w_t^i\| + \frac{\eta\coef}{n}\widehat{F}(w_t).
\end{align*}    
\begin{equation*}
\frac{1}{n}\sum_{i=1}^n\|w_{t+1} - w_{t+1}^i\|-\frac{1}{n}\sum_{i=1}^n\|w_{t} - w_{t}^i\| \leq \frac{\eta\coef}{n}\widehat{F}(w_t).
\end{equation*}
By summing until $T-1$, and another use of Jensen inequality we get,
\begin{align*}
    \frac{1}{n}\sum_{i=1}^n\|w_{T} - w_{T}^i\|
    &\leq 
    \frac{\eta\coef}{n}\sum_{t=1}^{T-1}\widehat{F}(w_t)
    \\
    &\leq
     \frac{\eta\coef}{n}\sum_{t=1}^{T}\widehat{F}(w_t)
    .
\end{align*}
 \end{proof}
 \begin{corollary}
 
\label{final exp stab power adapt smooth}
If $\ell(y)=e^{-y}$, GD on $\widehat{F}$ with step size $\eta=1$ and $w_1=0$ is $\ell_1$-on-average-loo model stable with
 \begin{equation*}
    \lipstab
     \leq O\left(\frac{\log^2T}{\margin^2n} \right)
     .
     \end{equation*}
 \end{corollary}